\documentclass{article}		
\PassOptionsToPackage{numbers,sort&compress}{natbib}
\PassOptionsToPackage{svgnames,dvipsnames}{xcolor}
\usepackage[margin=1.in,bottom=1.25in]{geometry}	
\usepackage{mathpazo}

\usepackage{amsmath}		
\usepackage{amssymb}		
\usepackage{amsfonts}		
\usepackage{amsthm}		
\usepackage{mathtools}		
\usepackage{xspace}         

\mathtoolsset{%
}

\usepackage[utf8]{inputenc}		
\usepackage[T1]{fontenc}		

\usepackage[
cal=cm,
]
{mathalfa}

\usepackage{dsfont}		

\usepackage{acronym}		
\renewcommand*{\aclabelfont}[1]{\acsfont{#1}}		

\usepackage[labelfont={bf,small},labelsep=colon,font=small]{caption}	
\usepackage{subcaption}		
\usepackage[dvipsnames,svgnames]{xcolor}		
\colorlet{MyRed}{FireBrick!50!Crimson}
\colorlet{MyBlue}{DodgerBlue!75!black}
\colorlet{MyGreen}{DarkGreen!85!black}
\colorlet{MyViolet}{DarkMagenta}

\colorlet{MyLightBlue}{DodgerBlue!20}
\colorlet{MyLightGreen}{MyGreen!20}

\colorlet{PrimalColor}{MyBlue}
\colorlet{PrimalFill}{MyLightBlue}
\colorlet{DualColor}{MyRed}

\colorlet{RevColor}{MyRed}
\colorlet{LinkColor}{MediumBlue}

\newcommand{\afterhead}{.}		
\newcommand{\ackperiod}{}		
\newcommand{\para}[1]{\paragraph{\textbf{#1\afterhead}}}

\usepackage{cancel}		
\usepackage{latexsym}		
\usepackage{wrapfig}
\usepackage{pifont}		

\newcommand{\needref}[1]{{\color{MyRed}\upshape\textbf{[??]}}\xspace}	

\usepackage{tikz}		
\usepackage{tikz-cd}		
\usetikzlibrary{calc,patterns}		
\usepackage{pgfplots}

\usepackage{array}		
\usepackage{booktabs}		
\usepackage[inline,shortlabels]{enumitem}		
\setenumerate{itemsep=\smallskipamount,topsep=\smallskipamount,left=\parindent}
\setitemize{itemsep=\smallskipamount,topsep=\smallskipamount,left=\parindent}

\usepackage[kerning=true]{microtype}		

\usepackage{tabto}		
\usepackage{xspace}		

\usepackage[sort&compress]{natbib}		

\bibpunct[, ]{[}{]}{,}{n}{,}{,}

\usepackage{hyperref}
\hypersetup{
colorlinks=true,
linktocpage=true,
pdfstartview=FitH,
breaklinks=true,
pdfpagemode=UseNone,
pageanchor=true,
pdfpagemode=UseOutlines,
plainpages=false,
bookmarksnumbered,
bookmarksopen=false,
bookmarksopenlevel=1,
hypertexnames=false,
pdfhighlight=/O,
urlcolor=LinkColor,linkcolor=LinkColor,citecolor=LinkColor,	
pdftitle={},
pdfauthor={},
pdfsubject={},
pdfkeywords={},
pdfcreator={pdfLaTeX},
pdfproducer={LaTeX with hyperref}
}

\usepackage[sort&compress,capitalize,nameinlink]{cleveref}		
\crefname{algo}{Algorithm}{Algorithms}
\crefname{assumption}{Assumption}{Assumptions}
\crefname{case}{Case}{Cases}

\usepackage{chngcntr}
\usepackage{apptools}
\AtAppendix{\counterwithin{lemma}{section}}
\AtAppendix{\counterwithin{definition}{section}}
\AtAppendix{\counterwithin{theorem}{section}}

\usepackage{algpseudocode}		

\usepackage{thmtools}		
\usepackage{thm-restate}		

\theoremstyle{plain}
\newtheorem{theorem}{Theorem}
\newtheorem*{inftheorem*}{Informal Theorem}		
\newtheorem{corollary}{Corollary}		
\newtheorem{lemma}{Lemma}		
\newtheorem{proposition}{Proposition}		

\newtheorem*{corollary*}{Corollary}		

\theoremstyle{definition}
\newtheorem{definition}{Definition}		
\newtheorem{assumption}{Assumption}		

\newtheorem{example}{Example}		
\newtheorem{fact}{Fact}
\newtheorem*{definition*}{Definition}		
\newtheorem*{assumption*}{Assumptions}		
\newtheorem*{example*}{Example}		

\makeatletter		
\newcommand{\asmtag}[1]{
 \let\oldtheassumption\theassumption
 \renewcommand{\theassumption}{#1}
 \g@addto@macro\endassumption{
   \addtocounter{assumption}{0}
   \global\let\theassumption\oldtheassumption}
 }
\makeatother

\theoremstyle{remark}

\newtheorem*{remark*}{Remark}		

\newcounter{proofpart}

\numberwithin{example}{section}		

\usepackage[showdeletions]{color-edits}		
\usepackage[normalem]{ulem}		

\newcommand{\debug}[1]{#1}		

\newcommand{\newmacro}[2]{\newcommand{#1}{\debug{#2}}}		
\newcommand{\newop}[2]{\DeclareMathOperator{#1}{\debug{#2}}}		

\DeclarePairedDelimiter{\braces}{\{}{\}}		
\DeclarePairedDelimiter{\bracks}{[}{]}		
\DeclarePairedDelimiter{\parens}{(}{)}		

\DeclarePairedDelimiter{\abs}{\lvert}{\rvert}		

\DeclarePairedDelimiterX{\setdef}[2]{\{}{\}}{#1:#2}		
\DeclarePairedDelimiterX{\inner}[2]{\langle}{\rangle}{#1,#2}		
\DeclarePairedDelimiterXPP{\exclude}[1]{\mathopen{}\setminus}{\{}{\}}{}{#1}		

\newcommand{\N}{\mathbb{N}}		
\newcommand{\R}{\mathbb{R}}		

\DeclareMathOperator{\bigoh}{\mathcal{O}}		
\DeclareMathOperator{\dist}{dist}		
\DeclareMathOperator{\grad}{\nabla}		
\DeclareMathOperator{\one}{\mathds{1}}		
\DeclareMathOperator{\relint}{ri}		

\newcommand{\ie}{i.e.,\xspace}		

\newcommand{\alt}[1]{#1'}		
\newcommand{\altalt}[1]{#1''}		

\newmacro{\dd}{\:d}		
\newcommand{\eps}{\varepsilon}		

\newmacro{\const}{c}		
\newmacro{\constalt}{c'}		
\newmacro{\Const}{\rho}		
\newmacro{\coefalt}{\mu}		
\usepackage{xparse}
\NewDocumentCommand{\coef}{O{\lambda}}{\debug{#1}}
\newmacro{\param}{\theta}		
\newmacro{\params}{\Theta}		

\newmacro{\pexp}{p}		
\newmacro{\qexp}{q}		
\newmacro{\rexp}{r}		

\newmacro{\radius}{r}

\newmacro{\beforestart}{0}		
\newmacro{\start}{0}		
\newmacro{\afterstart}{1}		
\newmacro{\running}{\start,\afterstart,\dotsc}		
\newmacro{\halfrunning}{1,3/2,2\dotsc}		

\newmacro{\run}{t}		
\newmacro{\runalt}{s}		
\newmacro{\runaltalt}{\tau}		
\newmacro{\nRuns}{T}		
\newmacro{\runs}{\mathcal{\nRuns}}		

\newmacro{\state}{x}		
\newmacro{\statealt}{y}		
\newmacro{\statealtalt}{z}		

\renewcommand{\next}[1][\state]{\debug{#1}_{\run+1}}		

\newmacro{\tstart}{0}		
\renewcommand{\time}{\debug{t}}		
\newmacro{\timealt}{s}		
\newmacro{\horizon}{T}		

\newmacro{\traj}{x}		
\newmacro{\trajalt}{y}		
\newmacro{\trajaltalt}{z}		

\newmacro{\flowmap}{\Theta}		
\DeclarePairedDelimiterXPP{\flowof}[2]{\flowmap_{#1}}{(}{)}{}{#2}		

\newop{\Nash}{NE}		
\newop{\CE}{CE}		
\newop{\CCE}{CCE}		
\newop{\NI}{NI}		

\newop{\brep}{br}		
\newop{\reg}{Reg}		
\newop{\preg}{\overline{Reg}}		
\newop{\val}{val}		

\newcommand{\eq}{\sol}		

\newmacro{\play}{i}		
\newmacro{\playalt}{j}		
\newmacro{\playaltlalt}{k}		
\newmacro{\nPlayers}{N}		
\newmacro{\players}{\mathcal{\nPlayers}}		

\newmacro{\pure}{\alpha}		
\newmacro{\purealt}{\beta}		
\newmacro{\purealtalt}{\gamma}		
\newmacro{\nPures}{A}		
\newmacro{\pures}{\mathcal{\nPures}}		

\newmacro{\loss}{\ell}		
\newmacro{\pay}{u}		
\newmacro{\payv}{v}		
\newmacro{\pot}{f}		

\newmacro{\game}{\mathcal{G}}		
\newmacro{\gamefull}{\game(\players,\points,\pay)}		

\newmacro{\fingame}{\Gamma}		
\newmacro{\fingamefull}{\Gamma(\players,\pures,\pay)}		

\newmacro{\gmat}{g}		
\newmacro{\gdist}{\dist_{\gmat}}
\newmacro{\mfld}{M}		
\newmacro{\form}{\omega}		

\newmacro{\tvec}{z}		
\newmacro{\uvec}{u}		
\newmacro{\basin}{\mathbb{B}}
\newmacro{\ball}{\basin}		
\newmacro{\sphere}{\mathbb{S}}		

\newmacro{\graph}{\mathcal{G}}
\newmacro{\vertices}{\mathcal{V}}
\newmacro{\edges}{\mathcal{E}}

\newmacro{\mat}{A}		
\newmacro{\matalt}{c}		
\newmacro{\hmat}{H}		

\newop{\row}{row}		
\newop{\col}{col}		

\newmacro{\ones}{\mathbf{1}}		
\newmacro{\eye}{I}		
\newmacro{\zer}{\mathbf{0}}		

\DeclarePairedDelimiter{\norm}{\lVert}{\rVert}		
\DeclarePairedDelimiterXPP{\dnorm}[1]{}{\lVert}{\rVert}{_{\ast}}{#1}		

\DeclarePairedDelimiterXPP{\onenorm}[1]{}{\lVert}{\rVert}{_{1}}{#1}		
\DeclarePairedDelimiterXPP{\twonorm}[1]{}{\lVert}{\rVert}{_{2}}{#1}		
\DeclarePairedDelimiterXPP{\supnorm}[1]{}{\lVert}{\rVert}{_{\infty}}{#1}		

\DeclarePairedDelimiterX{\braket}[2]{\langle}{\rangle}{#1,#2}		

\newmacro{\vecspace}{\mathcal{V}}		
\newmacro{\subspace}{\mathcal{W}}		

\newmacro{\coord}{i}		
\newmacro{\coordalt}{j}		
\newmacro{\coordaltalt}{k}		
\newmacro{\nCoords}{n}		
\newmacro{\dims}{\nCoords}		
\newmacro{\vdim}{\nCoords}		

\newmacro{\pvec}{z}		
\newmacro{\pvecalt}{r}		
\newmacro{\bvec}{e}		
\newmacro{\bvecs}{\mathcal{E}}		
\newmacro{\cvec}{b}     
\newmacro{\cvecalt}{d}     

\newmacro{\pspace}{\vecspace}		
\newmacro{\dspace}{\vecspace^{\ast}}		

\newmacro{\dvec}{\dpoint}		
\newmacro{\dbvec}{\eps}		

\newmacro{\dpoint}{y}		
\newmacro{\dpointalt}{\alt\dpoint}		
\newmacro{\dpointaltalt}{\altalt\dpoint}		
\newmacro{\dpoints}{\mathcal{Y}}		

\newmacro{\dstate}{Y}		
\newmacro{\dbase}{v}		

\newop{\Opt}{Opt}		
\newop{\Sol}{Sol}		
\newop{\gap}{Gap}	
\newop{\dualitygap}{Duality-Gap}		
\newop{\orcl}{Or}		

\newmacro{\tfun}{f}		
\newmacro{\obj}{f}		
\newmacro{\objalt}{g}		
\newmacro{\sobj}{F}		

\newmacro{\gvec}{g}		
\newmacro{\oper}{A}		
\newmacro{\vecfield}{v}		

\newcommand{\sol}[1][\point]{#1^{\ast}}		
\newcommand{\sols}{\sol[\points]}		
\newmacro{\solvec}{\vecfield(\sol)}		
\newmacro{\solpay}{\eq[\payv]}		

\newmacro{\signal}{V}		
\newmacro{\step}{\gamma}		
\newmacro{\learn}{\eta}		

\newmacro{\vbound}{G}		
\newmacro{\lips}{\ell}		
\newmacro{\strong}{\mu}		
\newmacro{\smooth}{\beta}		

\newop{\tspace}{T}		
\newop{\tcone}{TC}		
\newop{\dcone}{\tcone^{\ast}}		
\newop{\ncone}{NC}		
\newop{\pcone}{PC}		
\newop{\hull}{\Delta}		

\newmacro{\cvx}{\mathcal{C}}		
\newmacro{\subd}{\partial}		

\newmacro{\minmax}{\mathcal{L}}		

\newmacro{\minvar}{{\point_{1}}}		
\newmacro{\minvaralt}{\alt\minvar}		
\newmacro{\minvars}{\points_{1}}		
\newmacro{\minsol}{\sol[\minvar]}		

\newmacro{\maxvar}{\point_{2}}		
\newmacro{\maxvaaltr}{\alt\maxvar}		
\newmacro{\maxvars}{\points_{2}}		
\newmacro{\maxsol}{\sol[\maxvar]}		

\newop{\Eucl}{\Pi}		
\newop{\logit}{\Lambda}		
\newop{\dkl}{KL}		

\newmacro{\hreg}{h}		
\newmacro{\hconj}{\hreg^{\ast}}		
\newmacro{\breg}{D}		
\newmacro{\mprox}{P}		
\newmacro{\mirror}{Q}		
\newmacro{\fench}{F}		
\newmacro{\depth}{H}		
\newmacro{\hstr}{K}		
\newmacro{\hker}{\theta}		

\newmacro{\proxdom}{\points_{\hreg}}		
\newmacro{\proxdomi}{\points_{\hreg_{\play}}}		
\newmacro{\zone}{\mathbb{D}}		

\DeclarePairedDelimiterXPP{\proxof}[2]{\mprox_{#1}}{(}{)}{}{#2}		

\newmacro{\point}{x}		
\newmacro{\pointalt}{\alt\point}		
\newmacro{\pointaltalt}{\altalt\point}		
\newmacro{\points}{\mathcal{X}}		
\newmacro{\intpoints}{\relint\points}		

\newmacro{\base}{p}		
\newmacro{\basealt}{q}		
\newmacro{\basealtalt}{u}		

\newmacro{\open}{\mathcal{U}}		
\newmacro{\closed}{\mathcal{C}}		
\newmacro{\cpt}{\mathcal{K}}		
\newmacro{\nhd}{\mathcal{U}}		

\newop{\ex}{\mathbb{E}}		
\newop{\prob}{\mathbb{P}}		
\newop{\Var}{Var}		
\newop{\simplex}{\hull}		

\providecommand\given{}		

\DeclarePairedDelimiterXPP{\exof}[1]{\ex}{[}{]}{}{
\renewcommand\given{\nonscript\,\delimsize\vert\nonscript\,\mathopen{}} #1}

\DeclarePairedDelimiterXPP{\probof}[1]{\prob}{(}{)}{}{
\renewcommand\given{\nonscript\:\delimsize\vert\nonscript\:\mathopen{}} #1}

\DeclarePairedDelimiterXPP{\oneof}[1]{\one}{\{}{\}}{}{
\renewcommand\given{\nonscript\,\delimsize\vert\nonscript\,\mathopen{}} #1}

\newmacro{\sample}{\omega}		
\newmacro{\samples}{\Omega}		

\newmacro{\seed}{\theta}		
\newmacro{\seeds}{\Theta}		

\newmacro{\filter}{\mathcal{F}}		
\newmacro{\probspace}{(\samples,\filter,\prob)}		

\newmacro{\history}{\mathcal{H}}		

\newmacro{\event}{E}       
\newmacro{\eventalt}{H}       

\newmacro{\mean}{\mu}		
\newmacro{\sdev}{\sigma}		
\newmacro{\variance}{\sdev^{2}}		

\newmacro{\proper}{\tau}		

\newmacro{\error}{Z}		
\newmacro{\noise}{U}		
\newmacro{\bias}{b}		
\newmacro{\brown}{W}		

\newmacro{\serror}{\theta}		
\newmacro{\snoise}{\xi}		
\newmacro{\sbias}{\psi}		

\newmacro{\sbound}{\sigma}		
\newmacro{\bbound}{B}		
\newmacro{\noisepar}{\sdev}		
\newmacro{\noisevar}{\variance}		

\addauthor[Generic Author]{GA}{DarkGreen}

\addauthor[Angel]{AG}{DarkMagenta}
\addauthor[Yudong]{YC}{orange}
\addauthor[Qiaomin]{QX}{DarkBlue}

\newmacro{\op}{V} 
\newmacro{\solb}{\mathrm{R}} 
\newmacro{\growth}{L} 
\newmacro{\wqsmscale}{\mu} 
\newmacro{\wqsmshift}{\lambda} 
\newcommand{\steps}[1]{\gamma^{\text{\tiny #1}}}
\newcommand{\stepsalt}[1]{\alpha^{\text{\tiny #1}}}
\newmacro{\stepalt}{\alpha}

\newcommand{\cnstminor}[1]{c_1^{\braces{\text{\tiny #1}}}} 
\newcommand{\cnstminoralt}[1]{c_2^{\braces{\text{\tiny #1}}}} 
\newcommand{\cnstsminor}[1]{(c_1,c_2)^{\braces{\text{\tiny {#1}}}}} 
\newmacro{\rvar}{g}
\newmacro{\kyrt}{\delta_{\text{\tiny{ KYRT}}}}
\newmacro{\energy}{\mathcal{E}}
\newcommand{\meas}[1]{\nu^{\text{\tiny #1}}}
\newmacro{\irr}{\psi}
\newcommand{\distr}[2]{\pi_{#2}^{\braces{\text{\tiny #1}}}}
\newmacro{\set}{C}
\newmacro{\tf}{\phi}
\newcommand{\tconst}[1]{\rho^{\braces{\text{\tiny #1}}}}
\newcommand{\tconstalt}[1]{\kappa^{\braces{\text{\tiny #1}}}}
\newcommand{\law}{\text{law}}
\newmacro{\normal}{\mathcal{N}}
\newmacro{\thres}{\theta}
\newmacro{\borel}{\mathcal{B}}
\newcommand{\snorm}{\norm{\state_\time-\state^*}}
\newcommand{\ssnorm}{\norm{\state-\state^*}}
\addauthor[Man]{MV}{MediumBlue}

\newmacro{\pdist}{P}		
\newcommand{\ourtitle}{ Stochastic Methods in Variational Inequalities:\\
Ergodicity, Bias and Refinements 
}

\title{\ourtitle}
\author{
    \begin{tabular}[t]{c@{\extracolsep{5em}}c} 
    Emmanouil V. Vlatakis Gkaragkounis & Angeliki Giannou \\
    \textit{University of California, Berkeley} & \textit{University of Wisconsin–Madison} \\[2ex]
    Yudong Chen & Qiaomin Xie \\
 \textit{University of Wisconsin–Madison} & \textit{University of Wisconsin–Madison}
    \end{tabular}
    {\footnote{{Emails: \texttt{emvlatakis@berkeley.edu}, \texttt{giannou@wisc.edu}, \texttt{yudong.chen@wisc.edu}, \texttt{qiaomin.xie@wisc.edu}}}}
}

\newacro{LHS}{left-hand side}
\newacro{RHS}{right-hand side}
\newacro{iid}[i.i.d.]{independent and identically distributed}

\newacro{NE}{Nash equilibrium}
\newacroplural{NE}[NE]{Nash equilibria}

\newacro{DGF}{distance-generating function}
\newacro{KKT}{Karush\textendash Kuhn\textendash Tucker}
\newacro{SFO}{stochastic first-order oracle}

\newacro{WQSM}{weakly quasi strongly monotone}
\newacro{SGDA}{Stochastic Gradient Descent Ascent}
\newacro{SEG}{Stochastic Extra Gradient}
\newmacro{\pdf}{\mathrm{pdf}}

\begin{document}
\addtocontents{toc}{\protect\setcounter{tocdepth}{0}}

\date{} 
\maketitle
\allowdisplaybreaks		
\acresetall		

\begin{abstract}
%
%
For min-max optimization and variational inequalities problems (VIP) encountered in diverse machine learning tasks, Stochastic Extragradient (SEG) and Stochastic Gradient Descent Ascent (SGDA) have emerged as preeminent algorithms. Constant step-size variants of SEG/SGDA have gained popularity, with appealing benefits such as easy tuning and rapid forgiveness of initial conditions, but their convergence behaviors are more complicated even in rudimentary bilinear models. 
Our work endeavors to elucidate and quantify the probabilistic structures intrinsic to these algorithms. By recasting the constant step-size SEG/SGDA as time-homogeneous Markov Chains, we establish a first-of-its-kind Law of Large Numbers and a Central Limit Theorem, demonstrating that the average iterate is asymptotically normal with a unique invariant distribution for an extensive range of monotone and non-monotone VIPs. Specializing to  convex-concave min-max optimization, we characterize the relationship between the step-size and the induced bias with respect to the Von-Neumann's value. Finally, we establish that Richardson-Romberg extrapolation can improve proximity of the average iterate to the global solution for VIPs. Our probabilistic analysis, underpinned by experiments corroborating our theoretical discoveries, harnesses techniques from optimization, Markov chains, and operator theory.
\end{abstract}

\section{Introduction}
\label{sec:introduction}

Variational inequalities problem (VIP) is a versatile framework that incorporates a broad range of problems including loss minimization, min-max optimization, bilinear games and various fixed point problems. Many problems in machine learning, such as training Generative Adversarial Networks (GANs) \cite{goodfellow_2014_gan}, Actor-Critic methods \cite{pfau2016connecting}, multi-agent reinforcement learning \cite{zhang2021multi} and robust learning \cite{wen2014robust}, can be cast as VIPs. 

In many applications of VIP, one is  given only a stochastic oracle, typically constructed from finite data, that provides noisy access to the underlying operator. Various stochastic algorithms for VIP have been proposed and analyzed, with two prime examples being Stochastic Extragradient (SEG) \cite{juditsky2011solving} and Stochastic Gradient Descent Ascent (SGDA) methods~\cite{nemirovski2009robust}. It has been well recognized that convergence properties of stochastic VIP methods are more delicate than their deterministic and loss minimization counterparts. Nevertheless, much progress has been made in recent years, on both SEG \cite{mishchenko2020revisiting,kannan2019optimal,mertikopoulos2019learning,HIMM20,beznosikov2020distributed,gorbunov2022stochastic} and SGDA \cite{nemirovski2009robust,loizou2021stochastic,yang2020global,lin2020finite, beznosikov2023stochastic}. The  closely related stochastic gradient descent (SGD) method~\cite{gorbunov2020unified}, which can be viewed as a special case of SGDA, has an even larger and still growing literature. Classical results on these stochastic methods typically assume that a diminishing step-size is used, which allows for last-iterate convergence to the global solution~\cite{Polyak92-Avg,Douc2018,kushner2003-yin-sa-book,Benveniste12-sa-book}.

In this paper, we focus on the constant step-size variants of SEG and SGDA. Constant step-sizes are popular in practice, with several major benefits: the resulting algorithm is easy to tune with only a single parameter; it is insensitive to the initial condition, which is forgotten quickly; the algorithm makes substantial progress even in the first few iterations. Empirically, the use of constant step-size often leads to good performance in practical machine learning tasks and beyond.

\begin{wrapfigure}{r}{0.4\textwidth} 
\vspace{-1em}
  \begin{center}
    \includegraphics[width=0.25\textwidth]{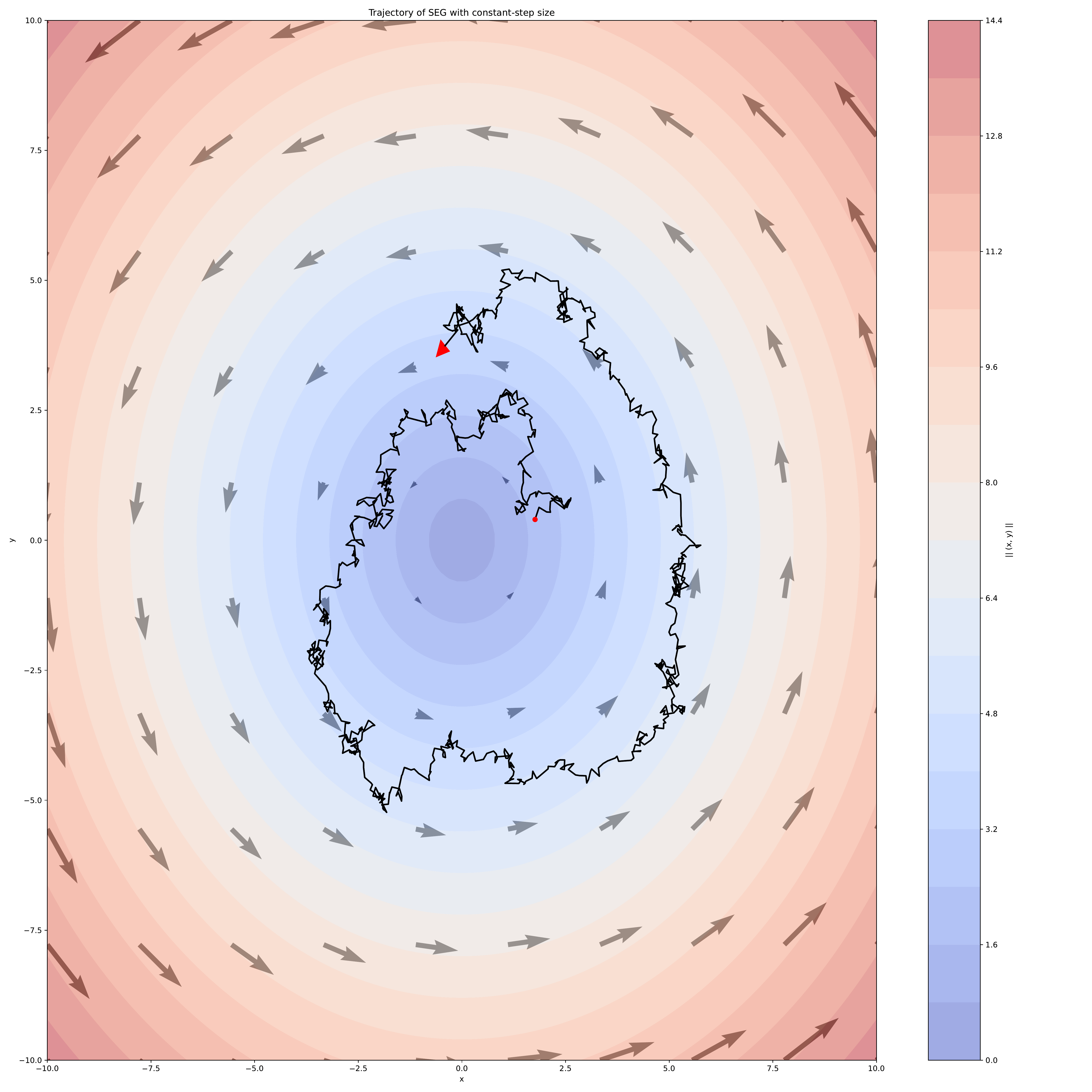} 
  \end{center}
\vspace{-1em}
\caption{\label{fig:divergent}Example of divergent behavior in a constant step-size SEG over a quasi-bilinear Game: $\min_x\max_y \epsilon x^2+xy-\epsilon y^2$, with $\epsilon\approx 10^{-4}$} 
\end{wrapfigure}
The analysis of constant step-size SEG and SGDA, however, is more complicated, with various non-convergent behaviors even in rudimentary bilinear models \cite{gidel2018variational,mertikopoulos2019optimistic,chavdarova2019reducing,HIMM20,daskalakis2018training}; see Figure~\ref{fig:divergent} for an example. In particular, similar to SGD \cite{Dieuleveut20-bach-SGD}, these algorithms in general do not converge to the exact solution of the VIP. Rather, due to stochastic noise, the iterates fluctuate within a neighborhood of the solution.  Existing theoretical results are typically in the form of an \emph{upper bound} on the mean squared error or dual gap of the iterations. Such upper bounds often compound the deterministic and stochastic aspects of the convergence behavior.

In this work, we seek to elucidate and quantify the behaviors of SEG and SGDA with constant step-sizes. Rather than treating the stochastic fluctuation as a nuisance, we fully embrace the probabilistic nature of SEG and SGDA. By viewing them as time-homogeneous Markov chains, we study their fine-grained distributional properties, disentangling the deterministic and stochastic components. In particular, we show that while the iterate does not converge, its distribution does. Moreover, this perspective allows us to use the information provided by their fluctuation for uncertainty quantification.

\para{Our contributions} 
We consider a class of VIPs with weak quasi strongly monotonicity, which encompasses a broad range of structured non-monotone and non-convex problems. Under appropriate regularity assumptions, we establish the following results.
\begin{itemize}[itemsep=0ex,leftmargin=*]
    \item We prove that the iterates of SEG and SGDA form a Harris and positive recurrent Markov chain, which admits a unique stationary distribution. Our results quantify the relationship between the step-size and the regularity parameters of the VIP for ensuring recurrence.
    \item We show that the distribution of the iterates converges geometrically to the above stationary distribution. More generally, we establish geometric convergence of the expectation of any Lipschitz test function of the iterates.
    \item We derive an ergodic Law of Large Number and a Central Limit Theorem for the iterates, thereby establishing the asymptotic normality of the ergodic average of the iterates.
    \item We show that induced bias---the distance between the mean of the stationary distribution and the global solution of the VIP---is bounded by a linear function of the step-size and the weak monotonicity parameter. Specializing to convex-concave min-max optimization, we quantify the relationship between the step-size and the bias with respect to the Von-Neumann’s value.
    \item For SGDA applied to quasi strongly monotone VIPs, we derive a first-order expansion of the induced bias in terms of the step-size. This characterization shows that an order-wise reduction the bias of SGDA can be achieved by the Richardson-Romberg refinement scheme. 
\end{itemize}

\para{Our challenges}
Firstly, solving Variational Inequalities (VIs) consists in principal more daunting than standard minimization tasks, primarily due to the absence of a clear potential function to measure closeness to the optimal solution value. 
Furthermore, the stochastic Extragradient method, usually employed for smooth operators, intensifies the complexity of the analysis due to the double random steps inherent in each iteration. The interdependence between these steps, where the first random occurrence directly impacts the subsequent one, necessitates more intricate maneuvering within a high-dimensional probabilistic landscape. Amid this, our study into Richardson Extrapolation propels the discourse beyond the conventional confines of co-coercive noisy gradient oracles, revealing a more nuanced proof under milder assumptions exclusively for the expected gradient. This meticulous analysis expands the realm of what was previously comprehended also in minimization tasks. Lastly, we strive to unify the stochastic analysis across minimization, min-max scenarios, and generic VIs, paving the way in future work towards a more comprehensive understanding of constrained case and different algorithms.
\

\para{Our techniques}

This research provides a novel proof that the average behavior of Stochastic Extragradient (SEG) and Stochastic Gradient Descent Ascent (SGDA) methods, with a constant step-size, will converge towards a typical trajectory over time, regardless of the initial conditions. By considering these methods as continuous-state Markov Chains, the study exploits Markov Chain Central Limit Theorems, Richardson extrapolation, and \citeauthor{Meyn12_book}'s machinery to validate the existence of an invariant probability measure, thereby confirming the ergodic behavior. This validation is realized through the application of non-uniform versions of  Doeblin's bound and the Foster-Lyapunov inequality within a well-defined "small set" around the solution set. Our study confirms that iterations will return to this small set infinitely many times, ensuring geometric convergence to a unique stationary distribution over time, regardless of the initial conditions.

\subsection{Related work}

Below we review prior work on VIP with a focus on stochastic methods with constant step-sizes.

\para{Variational Inequalities} VIP and its various special cases has been studied extensively, especially in the deterministic setting where one has exact access to the operator. Many algorithms have been developed, with both asymptotic convergence and finite-time guarantees. It is beyond the scope of this paper to survey these results, but we mention that for VIPs with Lipschitz
continuous and monotone operator, the works \cite{nemirovski2004prox} study a variant of Extra Gradient algorithm \cite{Kor76} and establishes optimal convergence rates for ergodic average, and the work \cite{gidel2018variational,mokhtari2020unified} studies proximal point algorithm with geometric
convergence results.

Most related to us are works for the stochastic setting, for which SEG \cite{juditsky2011solving}
and SGDA \cite{nemirovski2009robust} are two of the most prominent algorithms. Non-convergent phenomena are observed even in  unconstrained
bilinear games \cite{gidel2018variational,mertikopoulos2019optimistic,chavdarova2019reducing,daskalakis2018training,HIMM20}. 
Complementarily, a growing line of work has been dedicated to better
understanding of SEG and SGDA and bridging the gap
between the deterministic and the stochastic cases. The work \cite{juditsky2011solving}
provided the first analysis of SEG for monotone VIPs. Subsequent
work has extended these results to other settings \cite{mishchenko2020revisiting,kannan2019optimal,mertikopoulos2019learning,HIMM20,beznosikov2020distributed,gorbunov2022stochastic}.
A parallel line of work studies SGDA and its variants under different scenarios
\cite{nemirovski2009robust,loizou2021stochastic,yang2020global,lin2020finite}.
Recently \cite{beznosikov2023stochastic} proposed a unified convergence
analysis that covers various SGDA methods for regularized VIPs, where
the operator is either quasi-strongly monotone or $\ell$-star-cocoercive.
For a quantitative summary of existing results, we refer the readers to \cite{gorbunov2022stochastic} for SEG and \cite{beznosikov2023stochastic} for SGDA.

In this paper we consider \emph{weakly quasi-strongly monotone} VIPs, which is a class
of structured non-monotone operators under which one can bypass the
the intractability issue that arises in general non-monotone regime
\cite{diakonikolas2021efficient,daskalakis2021complexity,kakutani2hard}. Similar
conditions have been considered in prior work to establish the convergence
guarantee of various algorithms \cite{HIMM20,gorbunov2022stochastic,yang2020global,song2020optimistic,loizou2021stochastic}.

\para{Constant step-size SGD and Stochastic Approximation}

The literature on SGD and stochastic approximation (SA) is vast. Within this literature, our work is most related to, and in fact motivated by, a recent line of work that studies constant step-size SGD and SA through the lens of stochastic processes. The work \cite{Dieuleveut20-bach-SGD} studies SGD for smooth and strongly
convex functions. Extensions to non-convex functions are considered in 
\cite{Yu21-stan-SGD}, which establishes a central limit theorem that is similar in spirit to our results. More recently, \cite{bianchi2022convergence} studies SGD for non-smooth non-convex functions. 
The work \cite{durmus2021riemannian} considers constant step-size SA on Riemannian
manifolds and studies the limiting behavior as the step-size approaches zero. The work \cite{huo21-td} considers linear SA with Markovian noise; see the references therein for other recent results on SA. We mention that both \cite{Dieuleveut20-bach-SGD} and  \cite{huo21-td} examine the Richardson-Romberg bias refinement scheme, which we also consider in this paper.

\section{Problem setup}
\label{sec:prelims}


To provide a concrete foundation for our ensuing discussion, we first delineate the fundamental variational inequality framework that forms the backbone of our investigation in the subsequent sections. 

\subsection{Variational inequalities}

Let $\op:\R^d\to\R^d$ be a single-valued operator. The variational inequality problem related to the operator $\op$, when no constraints are involved, is:
\begin{equation}\tag{VI}\label{eq:vi}
    \text{Find } \point^*\in\R^d \text{ such that } \op(\point^*) = 0.
\end{equation}
Below, we provide a series of examples which showcase potential interpretations of the operator $\op$.

\begin{example}[Non-linear Systems of Equations]\label{ex:non-linear-systems} In this scenario, the operator $\op$ corresponds to the non-linear function $\mathbf{F}:\R^d\to \R^d$ that represents the system of equations. Formally, we write $\op = \mathbf{F}$. The solution of \eqref{eq:vi}, denoted as $\point^*$, is a root of $\mathbf{F}$, i.e., it satisfies $\mathbf{F}(\point^*) = \mathbf{0}$.
\end{example}

\begin{example}[Loss minimization]
\label{ex:loss-min}In this case the operation $\op$ corresponds to the gradient of a function that we try to minimize. Formally, we have $\op = \grad f$ for some smooth loss function $f:\R^d\to \R$. Then, the solution of \eqref{eq:vi}, $\point^*$, is a critical point of $f$, \ie $\grad f(\point^*) = 0$.
\end{example}

\begin{example}[Saddle-point problems]
\label{ex:min-max} Consider a smooth loss function $\minmax: \R^{d_1}\times \R^{d_2} \to \R$ which assigns a cost of $\minmax(\minvar,\maxvar)$ to a player choosing $\minvar\in\R^{d_1}$ and a payoff $\minmax(\minvar,\maxvar)$ to a player choosing $\maxvar\in\R^{d_2}$. Then, the saddle-point problem associated with a $\minmax$ aims to find $(x^*,y^*)$ such that
\begin{equation}
    \minmax(\minsol,\maxvar)\leq \minmax(\minsol,\maxsol) \leq \minmax(\minvar,\maxsol).
\end{equation}
The pair $(\minsol,\maxsol)$ is a saddle point of $\minmax$. 
With $\op = (\grad_{\minvar} \minmax, -\grad_{\maxvar} \minmax)$ the solutions of \eqref{eq:vi} correspond to critical points of $\minmax$, while if $\minmax$ is also convex-concave it corresponds to a saddle point.
\end{example}

The above examples represent a broad spectrum of applications:
\cref{ex:non-linear-systems} is related to Computational Fluid Dynamics and Physics, where Navier-Stokes or Maxwell equations encapsulate non-linear systems \cite{DBLP:journals/appml/Hao21}; \cref{ex:loss-min} is central to machine learning, reflecting model training via loss function minimization \cite{lan2020first}; \cref{ex:min-max} garners more and more attention due to developments in GANs \cite{goodfellow_2014_gan,gidel2018variational,daskalakis2018training}, Actor-Critic methods \cite{pfau2016connecting}, and multi-agent Reinforcement Learning \cite{zhang2021multi}.

\subsection{Assumptions} 

Our blanket assumptions concerning the operator $\op$ are the following:
\begin{assumption}
\label{asm:solutionset}
The set of solutions $\sols$ of \eqref{eq:vi} is non-empty and $\exists \point^* \in \sols, \solb\in\R$ such that $\norm{\point^*}\leq \solb$.
\end{assumption}

\begin{assumption}\label{asm:wqsm}
    The operator is $\wqsmshift$-weak $\wqsmscale$-quasi strongly monotone with $\wqsmshift\geq 0$, $\wqsmscale >0$ , \ie
    \begin{equation}\label{eq:wqsm}
        \inner{\op(\point)}{\point -\point^*} \geq \wqsmscale\norm{\point -\point^*}^2 - \wqsmshift \text{ for all } \point\in\R^d \text{ and some }\point^*\in\sols.
    \end{equation}
 
\end{assumption}
\begin{remark*} 
Notice that \cref{eq:wqsm} implies directly that $\norm{\point^*_1-\point^*_2}^2 \leq \frac{\wqsmshift}{\wqsmscale}$ for any $\point^*_1,\point^*_2\in\sols$.
Thus, \Cref{asm:wqsm} yields that $\sols$ is actually contained in some ball of radius $\sqrt{\frac{\wqsmshift}{\wqsmscale}}$.
\end{remark*}

Our next assumption pertains to the two algorithms \acl{SGDA} (SGDA) and the \acl{SEG} (SEG), which are formally given in Section \ref{sec:algorithms}. Conforming to the customary convention in variational inequality literature, we make the presumption that when SEG is employed, we are dealing with a Lipschitz operator (so-called smooth case), while SGDA is used in scenarios that exhibit just linear growth (so-called non-smooth case).
\begin{assumption}
\label{asm:lipschitz-lineargrowth}
Unless we state it differently, we adopt the following convention for the Lipschitzness/bounded growth of the operator for different algorithms respectively:
\begin{itemize}[noitemsep,nolistsep,leftmargin=20pt]
    \item If \eqref{eq:SEG} is run, we have that 
    the operator $\op$ is $\lips$-Lipschitz continuous, \ie
    \begin{equation}
        \norm{\op(\pointalt) -\op(\point)}\leq \lips\norm{\pointalt -\point} \text{ for all }\point,\pointalt \in\R^d.
    \end{equation}
\item If \eqref{eq:SGDA} is run, we have that the operator $\op$ has at most $\growth$-linear growth, \ie 
    \begin{equation}
        \norm{\op(\point)} \leq \growth(1+\norm{\point}) \text{ for all } \point\in\R^d .
    \end{equation}
\end{itemize}
\end{assumption}

\begin{assumption}\label{asm:oracle} 
In the ensuing discussion, we presuppose that our algorithms have access to $\op$ at each stage $\run\geq \start$ through a stochastic oracle. Specifically, at each iteration $\run$, the algorithm can pick a point $\state_\run$ and call a black-box procedure that returns
\begin{equation}
\op_\run =\op(\state_\run) + \noise_\run(\state_\run).
\end{equation}
Here, $(\noise_\run(\cdot))_{\run\geq 0}$ is a sequence of \acl{iid} random fields that satisfy the following conditions: there exists a  filtration (denoting the history of $\state_\run$) $(\filter_\run)_{\run\geq 0}$ on a certain probability space $(\Omega,\filter,\prob)$, such that $\noise_\run(\state_\run)$ is $\filter_{\run+1}-$measurable, but not $\filter_\run-$measurable and corresponds to a noise with $(i)$  \emph{Zero mean}: $\exof{\noise_\run(\state)\given\filter_{\run}} = 0$
and $(ii)$
 \emph{Bounded second moment}: $\exof{\norm{\noise_\run(\state)}^2\given\filter_\run} \leq \sbound^2$ for all $\state\in\R^d$ and some constant $\sbound>0$.

 \medskip

\emph{Additional remarks on the above assumptions:}
\cref{asm:solutionset} is standard and widely adopted in the literature on VIP.
\cref{asm:wqsm} represents a further relaxation of $\mu$-quasi strongly monotonicity, inspired by weakly dissipative dynamical systems and weakly convex optimization~\cite{erdogdu2018global,raginsky2017non}. This assumption is inclusive of special cases of non-monotone games. It is worth mentioning that for $\lambda>0, \mu>0$, it could encompass functions of the form $a_{\lambda,\mu}\|x\|^2+ b_{\lambda,\mu} \sin(\|x\|)$, as well as rescaled versions of the Rastrigin function or various non-monotone operators frequently encountered in statistical learning \cite{DBLP:journals/corr/abs-1910-12837}. In the context of $\lambda=0$, this assumption has been explored in the literature of VIPs under various names, e.g., quasi-strongly monotone problems \cite{loizou2021stochastic}, strong coherent VIPs \cite{song2020optimistic}, or VIPs satisfying the strong stability condition \cite{mertikopoulos2019learning}.
\cref{asm:lipschitz-lineargrowth} corresponds to a well-established dichotomy on VIPs: we leverage \eqref{eq:SEG} for its superior rates in smooth optimization scenarios, whereas \eqref{eq:SGDA} is employed in cases of non-smooth optimization.
Finally, \cref{asm:oracle} is standard for the analysis of stochastic algorithms in VIPs and optimization~\cite{nemirovski2009robust,mertikopoulos2019learning,yang2020global,hsieh2019convergence,HIMM20}.
\end{assumption}

\section{Algorithms}
\label{sec:algorithms}

In this paper we focus on two of the most widely used algorithms for variational inequalities:  \acl{SGDA} (SGDA) and \acl{SEG} (SEG). 

\textbf{\acl{SGDA}}. At each time-step $\run\in\N,$ a vector $\state_\run\in\R^d$ is  maintained and updated by accessing the stochastic oracle $\op_\run$, using a constant step-size $\steps{SDGA}\in (0,\infty)$. Formally,
\begin{equation}\label{eq:SGDA}\tag{SGDA}
   \state_{\run+1} = \state_\run -\steps{SGDA}\op_\run= \state_\run - \steps{SGDA}(\op(\state_\run)+\noise_\run(\state_\run)),
\end{equation}
where $\op$ and $(\noise_\run)_{\run \geq 0}$ satisfy \cref{asm:lipschitz-lineargrowth,asm:wqsm,asm:oracle}.

\textbf{Double Step-size \acl{SEG}}. As previously delineated, the preferred approach for smooth variational inequality problems is the stochastic variants of the extragradient (EG) algorithm of Korpelevich \cite{Kor76}, where at each step it uses an extra gradient "look-ahead" step $\op_{\run+1/2}$ to enhance convergence towards the solution. Formally, the  incarnation of SEG with double constant step-size $(\stepsalt{SEG},\steps{SEG})$ can be defined as follows:
\begin{equation}\label{eq:SEG}\tag{SEG}
\state_{\run+1/2} = \state_\run -\steps{SEG}\op_\run,  \qquad\qquad \state_{\run+1} = \state_{\run} - \stepsalt{SEG}\steps{SEG}\op_{\run+1/2},
\end{equation} 
where $\op$ and $(\noise_\run,\noise_{\run+1/2})_{\run\geq \start}$ satisfy \cref{asm:lipschitz-lineargrowth,asm:wqsm,asm:oracle} with intermediate step filtration satisfying $\filter_{\run +\frac{1}{2} }= \filter_{\run}$.

Inspired by seminal work on stochastic gradient descent \cite{Dieuleveut20-bach-SGD}, here we 
study the trajectories of both \eqref{eq:SGDA} and \eqref{eq:SEG} via the lens of Markov Chain theory.
Indeed, their iterates $(\state_\run)_{\run\geq 0}$  can be cast as time-homogeneous continuous Markov chains in $\R^d$.

Specifically, observe that:
\begin{enumerate}[noitemsep,nolistsep,leftmargin=*,label=(\roman*)]
\item The iterates $(\state_\run)_{\run\geq 0}$ of \eqref{eq:SGDA} and \eqref{eq:SEG} constitute respectively a Markov chain:  the subsequent state $\state_{\run+1}$ (post-update parameters) relies solely on the current state $\state_\run$.
\item The chain is time-homogeneous, meaning the transition kernel does not depend on time:
this is attributed to the constant  step-size in the update rule applied at each step with i.i.d.\ random fields $(\noise_\run(\state))_{\run\geq 0}$. 
\item The chains lie in the general continuous state space $\R^d$, in contrast to the typical discrete ones.
\end{enumerate}
For a formal proof of the above claims, we direct interested readers to our appendix. 
In parallel to the study of Markov chains in a discrete finite state space, our analysis in the continuous state space primarily focuses on three fundamental properties: \emph{irreducibility}, \emph{aperiodicity}, and \emph{recurrence}~\cite{Meyn12_book}.  
Building on these three properties, we establish limit theorems that shed light on the long-run behavior of the chains. The forthcoming sections aim to grapple with the amplified challenges that arise due to our chain trajectories navigating through multi-dimensional, uncountable domains.

\subsection{Convergence up to constant factors}
\label{sec:conv_up_to_constant}

We begin by deriving a basic convergence result that resembles the classical descent inequalities. This result serves as a robust tool for understanding the recurrent behavior of our chains.

As established in prior work \cite{gorbunov2022stochastic,beznosikov2023stochastic} and highlighted in the introduction, when the operator $\op$ is Lipschitz and strongly monotone, the full-information/noiseless equivalent of SGDA/SEG
attain exponential rate of convergence to some solution in the solution set $\sols$. By relaxing the assumption of strong monotonicity to the assumption of weakly quasi strong monotonicity (\cref{asm:wqsm}), we show that this result can be achieved in the noisy setting as well up to an additive constant. The cornerstone of our proof hinges on the construction of a quasi-descent inequality \cite{loizou2020stochastic} and the appropriate determination of a step-size in order to account for both the variance $\sbound^2$ and the shift $\wqsmshift$ of weakly quasi-monotonicity.
The additive constant factor corresponds to the bias introduced by the stochasticity and non-monotonicity of $\op$, and it depends on the constant step-sizes $\steps{SGDA,SEG},\stepsalt{SEG}$ used in the respective algorithms. 
 
Formally, the following theorem holds:
 
 \begin{theorem}\label{thm:convergence}
 Consider that either \eqref{eq:SGDA}  or \eqref{eq:SEG} is run with a stochastic oracle satisfying \cref{asm:solutionset,asm:lipschitz-lineargrowth,asm:wqsm,asm:oracle} respectively with step-sizes $\steps{SGDA} < \dfrac{\wqsmscale}{\growth^2}$, $\steps{SEG}< \dfrac{1}{2\wqsmscale+\sqrt{3}\lips}$ and $\stepsalt{SEG}\in(0,1)$ and let $(\state_\run)_{\run\geq \start}$ be the iterations generated. Then, there exists a pair of constants\footnote{For the explicit formula of the constants, we refer the reader to the proof at the supplement.} $\cnstsminor{SGDA,SEG}$ that depend on the choice of step-sizes, as well as the parameters of the model, with $\cnstminor{SGDA,SEG}\in(0,1)$ and $\cnstminoralt{SGDA,SEG}\in(0,+\infty)$ such that 
 \begin{equation}
    \exof{\norm{\state_{\run+1}-\state^*}^2}  \leq \parens*{1-\cnstminor{SGDA,SEG} }^\run\norm{\state_{\start}-\state^*}^2 + \cnstminoralt{SGDA,SEG}, 
 \end{equation}
 for any initial point $\state_0\in\R^d$.
 \end{theorem}

A byproduct of the above theorem's proof is the following one-step ``quasi-descent'' inequality:
 \begin{corollary}\label{cor:energy} 
 Under the conditions of \cref{thm:convergence},  for all  $\state^*\in\sols$ there exists an extended real-valued function $\energy:\R^d\to[1,\infty]$ and constants $\cnstminor{SGDA,SEG}\in(0,1),\cnstminoralt{SGDA,SEG}\in(0,\infty)$ such that  
 \begin{equation}
    \exof{\energy(\state_{\run+1},\state^*)  \given\filter_\run}\leq  \cnstminor{SGDA,SEG}\energy(\state_\run,\state^*) + \cnstminoralt{SGDA,SEG}.
 \end{equation}
 Specifically, $\energy(\state_\run,\state^*) = \norm{\state_{\run}-\state^*}^2 +1$. 

 \end{corollary}
\begin{remark*}
    The function $\energy$ is sometimes called an energy, potential or Lyapunov function.
    While the above corollary applies to any $\state^*\in\sols$, for the sake of conciseness, we will assume a fixed but arbitrary $\state^*$ and omit its reference. From now on, we will simply write the energy function as $\energy(\state_\run)$.
\end{remark*}
 
Understanding Markov chains in continuous domains requires a grasp of different types of recurrences: \emph{(null)-recurrence}, \emph{Harris recurrence}, and \emph{positive recurrence}, each progressively contributing to our insights on the chain behavior. Recurrence indicates a state will infinitely visit nearby regions on expectation, but without timing guarantees. Harris recurrence, specific to continuous state space Markov chains, ensures a state revisits the nearby areas infinitely often almost surely. Positive recurrence, an orthogonal refinement, promises a state's recurrent visits within a finite expected time. 
(For their formal definitions, we refer to our introductory appendix on Markov Chains.)

Harris and positive recurrence are the pivotal properties that underpin our key results on the existence of   $(a)$ an invariant measure, $(b)$ a law of large numbers, and $(c)$ an ergodic central limit theorem.

\section{Main Results}
\label{sec:results}

The main result of this section can be summarized as follows:
\begin{inftheorem*}[Main Result]
Under \cref{asm:solutionset,asm:lipschitz-lineargrowth,asm:wqsm,asm:oracle}, the Stochastic Extragradient \acl{SEG} and Stochastic Gradient Descent Ascent \acl{SGDA} methods with  constant step-size, behave as strong aperiodic, positive Harris recurrent continuous-state Markov Chains, converging to a unique stationary distribution over time regardless of the initial conditions. Moreover, their trajectory's ergodic averages adhere to the Law of Large Numbers and the Central Limit Theorem.
\end{inftheorem*}

\para{Proof Sketch}
Our main objective is to showcase that, under constant step-size, the average trajectory of SEG and SGDA methods converges to a typical path over time, validating their ergodic behavior. This endeavor necessitates the fusion of optimization and probabilistic techniques.

Our investigation commences by observing that both SEG and SGDA methods, when operating under a constant step-size, behave akin to continuous-state Markov Chains within the Euclidean space $\mathbb{R}^d$. To further exploit machinery such as Markov Chain Central Limit Theorems, Richardson extrapolation, etc., our primary objective is to ascertain the existence of an invariant probability measure. We achieve this by establishing properties like strong aperiodicity, positive Harris recurrence, and irreducibility—paralleling the standard approach for finite discrete-state Markov chains. Our proof for these properties leans heavily on a single-step probability minorization condition and arguments based on Lyapunov potential functions. In addition, the application of the SEG method to VIs brings added complexities due to its intricate update rule, contrasting the simpler case of Stochastic Gradient Descent (SGD) used for minimization task.

Focusing on our techniques, we extensively use a version of Doeblin’s bound. In words this minorization condition posits  that from any state, there's a positive probability that the chain will transition into a designated subset of states within one step. In mathematical terms, for all $x \in S$ and for all measurable subsets $A \subseteq S$ (where $S$ is the state space), there's a positive probability that $P(x, A)$ is at least $\epsilon \cdot \mu(A)$ for some $\epsilon > 0$ and a probability distribution $\mu(\cdot)$. We then construct a coupling for two probability laws: $Z_1$ distributed according to $\nu(x) \cdot P^n(x, \cdot)$ and $Z_2$ according to $\pi(x) \cdot P^n(x, \cdot)$, for any arbitrary  $x \in S$ and the stationary distribution $\pi(\cdot)$. This guarantees that the total variation distance between the laws of $Z_1,Z_2$ is bounded by $(1 - \epsilon)^n$ for any $\nu$ probability measure.

While in discrete settings we could consider the entire state space, it is not feasible to do so in continuous domains like $\mathbb{R}^d$. We navigate this challenge by applying the minorization condition within a bounded region around the solution set, referred to as $S^*:=\text{Ball}(X^*, r^*)$. Such regions are termed "small sets" in the literature of Markov Chains. In the context of Markov Chains literature, such regions are commonly referred to as "small sets".  Given a state $x$ that resides within $S^*$, Doeblin’s condition ensures geometric convergence to the invariant probability. To extend this convergence rate to $\mathbb{R}^d$, we employ the Foster-Lyapunov (FL) inequality within a well-tailored small set.
FL inequality \textemdash also known geometric drift property (See \cite{TW03})\textemdash ensures that the distance from the solution set remains bounded in expectation and diminishes according to a quasi-descent inequality if the current state resides within a judiciously chosen attraction region. Using this inequality, we establish that iterations outside a small set $S$ will converge on expectation to $S$ exponentially fast, suggesting infinite visits to $S$ and affirming geometric convergence to a unique stationary distribution,
independent of the initial state.

In order to employ our stochastic analysis toolkit, we embrace the following standard regularity assumption regarding the nature of the noise \cite{Yu21-stan-SGD}.

\begin{assumption}\label{asm:uniformbound}
The random variable $\noise_\time(x)$ can be decomposed as $\noise_\time(x) = \noise_\time^a(x) + \noise_\time^b(x)$, such that the probability distribution of $\noise_\time^a(x)$ has a density function, $\pdf_{\noise_\time^a(x)}$, with respect to the Lebesgue measure satisfying $\inf_{x\in C}\pdf_{\noise_\time^a(x)}(t) >0$ for all bounded sets $C\subseteq \R^d$ and for all $t\in\R^d$.
\end{assumption}

Regarding the applicability of this assumption, observe that any Gaussian random field, among others, satisfies Assumption~\ref{asm:uniformbound}.
\subsection{Minorization Condition,  Geometric Drift Property \& Recurrence Classification}
Inspired by the Markov chain stability framework in \cite{Meyn12_book}, we prove two important properties: the \emph{Minorization Condition} and the \emph{Geometric Drift Property}. Both of them serve an important role in proving Harris and Positive Recurrence respectively.
\begin{lemma}\label{lem:minor}
   Let the assumptions \cref{asm:solutionset,asm:lipschitz-lineargrowth,asm:wqsm,asm:oracle,asm:uniformbound} be satisfied for \eqref{eq:SGDA} and \eqref{eq:SEG}. Then given the step-sizes specified in \cref{thm:convergence},  both algorithms satisfy the following minorization condition: there exist a constant $\delta>0$, a probability measure $\nu$ and a set $C$ dependent on the algorithm, such that $\nu(C)=1$, $\nu(C^c)=0$ and
\begin{equation} \label{eq:main:minor}
    \Pr[\next\in A|\state_\run=x]\geq \delta\one_{\set}(x)\nu(A) 
    \quad \text{ for all}\;\;A\in\borel(\R^d),\  x\in \R^d.
\end{equation}
\end{lemma}
If the set $C$ encompassed the entire space, \cref{eq:main:minor} would indicate that every subspace of $\mathbb{R}^d$ is reachable from any state. This would lead, through standard coupling arguments, to geometric convergence of the distribution of $\state_\run$ towards a unique distribution. Although this scenario may not hold in our unbounded state space, a subset $C$ that satisfies this condition, known as a "small/petite" set, can still ensure geometric convergence if a Foster-Lyapunov drift property is satisfied.
\begin{corollary}\label{cor:geomdrift}
    Under the setting of \cref{lem:minor}, the function $\energy:\R^d\to \R$ presented in \cref{cor:energy} satisfies the following geometric drift property by \eqref{eq:SGDA} or \eqref{eq:SEG}: there exists  a measurable set $C$, and constants $\beta >0$, $b<\infty$ such that 
    \begin{equation}\label{eq:main:drift}
        \Delta \energy(x) \leq -\beta \energy(x) + b\one_C(x), x\in \R^d,
    \end{equation}
    where $\Delta \energy(x) = \int_{y\in \R^d} P(z,dy)\energy(y) - \energy(x)$.
\end{corollary}

The above property is called the (V4) geometric drift property in \cite{Meyn12_book}.
In simple terms, the Foster-Lyapunov inequality \eqref{eq:main:drift} controls how quickly the energy function decreases as the Markov chain transitions between states. If r.h.s.\ of \eqref{eq:main:drift} is negative, it indicates an exponential rate of decrease, which in turn implies that the chain ``forgets'' its initial state and exhibiting predictable and stationary behavior around minimum of our energy function $\energy(\cdot)$.

Equipped with the Minorization condition and the geometric drift property, we are ready to show all the necessary conditions for proving the ergodicity of \eqref{eq:SGDA} and \eqref{eq:SEG}. Specifically,
\begin{lemma}\label{lem:properties}
    The Markov chain sequences $(\state_\run)_{\run\geq \start}$ corresponding to \eqref{eq:SGDA} and \eqref{eq:SEG} have the following properties:
    \begin{itemize}[noitemsep,nolistsep,leftmargin=*]
        \item They are $\irr-$irreducible for some non-zero $\sigma$-finite measure $\irr$ on $\R^d$ over Borel $\sigma$- algebra of $\R^d$.  
        \item They are aperiodic. 
        \item They are Harris and positive recurrent with an invariant measure.
    \end{itemize}
\end{lemma}

Thus using generalizations of aperiodic ergodic theorem for Markov chains satisfying the geometric drift property, we prove our first main result about the invariance measure. In the following, we let $\mathcal{P}_2(\R^d):=\{\nu: \int_{\R^d}\|x\|^2\nu(dx)<\infty \}$ denote the set of square-integrable probability measures.
\subsection{Invariant Measure, Law of Large Numbers \& Central Limit Theorem}
\begin{theorem}\label{thm:close}
    Let \cref{asm:solutionset,asm:lipschitz-lineargrowth,asm:wqsm,asm:oracle,asm:uniformbound} be satisfied for \eqref{eq:SGDA} and \eqref{eq:SEG}. Then given the step-sizes specified in \cref{thm:convergence}, it holds that
    \begin{enumerate}
        \item \eqref{eq:SGDA} and \eqref{eq:SEG} iterates admit a unique stationary distribution $\distr{SGDA,SEG}{\step}\in\mathcal{P}_2(\R^d).$
        \item For each test function $\tf:\R^d\to\R$ satisfying that $\abs{\tf(\point)}\leq \growth_{\tf} (1+ \norm{\point})$ for all $\point\in\R^{d}$ 
        and some $\growth_\tf >0$ and for any initialization $\state_\start\in\R^d$, there exist $\tconst{SGDA,SEG}_{\tf,\step}\in(0,1)$ and $\tconstalt{SGDA,SEG}_{\tf,\state_\start,\step}\in(0,\infty)$ such that:
        \begin{equation}
            \label{eq:convergence_rate}
           \abs*{\mathbb{E}_{\state_{\run}}\bracks{\tf(\state_\run)} - \mathbb{E}_{\point\sim\distr{SGDA,SEG}{\step}}\bracks{\tf(\point)} }\leq \tconstalt{SGDA,SEG}_{\tf,\state_\start,\step}(\tconst{SGDA,SEG}_{\tf,\step})^\run.
        \end{equation}
        Hence, \eqref{eq:SGDA} and \eqref{eq:SEG} converges geometrically under the total variation distance to $\distr{SGDA,SEG}{\step}$.
    \item For each test function $\tf$ that is $\lips_\tf$-Lipschitz, it holds that
    \begin{equation}
       \abs{ \mathbb{E}_{\point\sim\distr{SGDA,SEG}{\step}}\bracks{\tf(\point)}-\tf(\state^*)} \leq \lips_\tf \sqrt{D^{\{\text{\tiny SGDA,SEG}}\}},
       \label{eq:bias_upper_bound}
    \end{equation}
    for some constant $D^{\{\text{\tiny SGDA,SEG}\}}\propto \max(\lambda,\steps{SGDA,SEG})/\mu$.  
    \end{enumerate}
\end{theorem}

The result outlined above provides critical insights into the behavior of constant step size Stochastic Extragradient \acl{SEG} and Stochastic Gradient Descent Ascent \acl{SGDA} methods. Notably, it asserts the uniqueness of the stationary distribution of these methods, assuming it has a bounded second moment. It further offers an analysis of the fluctuation patterns of a test function $\tf$ across the \acl{SEG}/\acl{SGDA} iterations, even in the face of non-smooth and non-convex objective functions.
Elaborating on the convergence properties, the theorem elucidates that the \acl{SEG}/\acl{SGDA} algorithm, irrespective of its initial point and provided the step size is suitably small, will gravitate towards its invariant distribution at an exponential rate (See \cref{eq:convergence_rate}). This effectively confirms the robustness of these algorithms under various initialization scenarios and across a wide spectrum of step sizes.
Lastly, for the class of smooth test functions, (See \cref{eq:bias_upper_bound}) the above result constrains the deviation of the expected value of the test function's asymptotic behavior from its optimal value, offering an explicit bound. This bound delineates a 'ball of interest', providing a tangible limit to the bias, thus enhancing our understanding of the overall performance of these algorithms.


Following the influential work of Polyak and Juditcky \cite{polyak90_average}, and having confirmed the uniqueness of the stationary distribution, we now focuses on the question of asymptotic normality of the two algorithms. To the best of our knowledge, such a result would be the first of its kind for stochastic approximation methods within the variational inequalities framework, especially for extrapolation techniques like \eqref{eq:SEG}.
Establishing such results allows us to provide theoretical guarantees when constructing confidence intervals in game scenarios, surpassing the sole dependence on empirical evidence, \ie \cite[Section 7]{ABM21,HIMM20}. To streamline our discussion, let us introduce a notation for any given function $\tf$:
\begin{definition}
    We denote the average iterate of our methods, also known as the C\'{e}saro mean \cite{hardy1992providence},  evaluated over a given function $\tf$ as $\overline{S_\nRuns(\tf)}:=
    \frac{1}{\nRuns}{S_\nRuns(\tf)}:=\frac{1}{\nRuns}\sum_{\run=\start}^\nRuns \tf(\state_\run)$. 
\end{definition}

Our inquiry begins with establishing a Law of Large Numbers (LLN) for \eqref{eq:SGDA} and \eqref{eq:SEG}.
By employing the analogue of the Birkhoff–Khinchin ergodic theorem  for continuous state space ergodic Markov Chains, we can derive the ensuing LLN:
\begin{theorem}\label{cor:LLN}
 Let the \cref{asm:solutionset,asm:lipschitz-lineargrowth,asm:wqsm,asm:oracle,asm:uniformbound} hold. Then for the choice of step-sizes specified in \cref{thm:close} and any function $\tf$ satisfying $\pi_{\step}(\abs{\tf}) <\infty$, where $\pi_{\step}(\abs{\tf}) = \ex_{\point\sim \distr{SGDA,SEG}{\step}}[\abs{\tf(x)}]$, it holds that 
     \begin{equation}
       \lim_{\nRuns\to\infty} \frac{1}{\nRuns}S_{\nRuns}(\tf)=\lim_{\nRuns\to\infty}\dfrac{1}{\nRuns}\sum_{\run=\start}^\nRuns \tf(\state_\run)  =\pi_{\step}(\tf) \quad\text{a.s.}
       \tag{Law of Large Numbers for \eqref{eq:SGDA},\eqref{eq:SEG}}
     \end{equation}
\end{theorem}

We next state a central limit theorem (CLT) for the sequences generated by \eqref{eq:SGDA} and \eqref{eq:SEG}, establishing the asymptotic normality of their averaged iterates:
 \begin{theorem}\label{thm:CLT}
    Let the \cref{asm:solutionset,asm:lipschitz-lineargrowth,asm:wqsm,asm:oracle,asm:uniformbound} hold. Then for the choice of step-sizes and a test function $\tf$ specified in \cref{thm:close}, we have that 
    \begin{equation}
    \nRuns^{-1/2}S_\nRuns(\tf - \pi_{\step}(\tf)) \xrightarrow{d}
 \normal(0,\sdev^2_{\pi_{\step}}(\tf)),
    \tag{Central Limit Theorem for \eqref{eq:SGDA},\eqref{eq:SEG}}
    \end{equation}
    where  $\pi_{\step}(\tf)= \ex_{\point\sim \distr{SGDA,SEG}{\step}}[\tf(x)]$ and $\sdev^2_{\pi_{\step}}(\tf):= \lim_{\nRuns\to\infty}\frac{1}{\nRuns}\ex_{\distr{SGDA,SEG}{\step}}\bracks{S_{\nRuns}^2(\tf-\pi_{\step}(\tf))}.$
     where $\ex_{\distr{SGDA,SEG}{\step}}$ denotes that the initial distribution of the Markov chain is $\distr{SGDA,SEG}{\step}$.
\end{theorem}

\section{Applications and Experiments}
\label{sec:application}

In this section, we discuss the applications of our main theoretical results. We will focus our examination on two interesting subcategories of quasi-strongly monotone problems: $(i)$ min-max convex-concave games, with locally quadratic region of attractions around the Nash Equilibria and $(ii)$ the application of Richardson-Romberg (RR) bias refinement scheme for smooth quasi-strongly monotone operators. While the region of attraction in the first instance could potentially be an artifact of our analysis, it is noteworthy that the application of RR presupposes the existence of a unique solution to be viable. We conclude the section by presenting a series of experiments validating our theoretical establishments.

\subsection{Min-Max Convex-Concave Games}

We now explore a specific class of operators that lie in the merely monotone regime:
\begin{assumption}
\label{asm:monotone}
we assume that the operator $\op$ is monotone in the sense that 
\begin{equation}
    \inner{\op(\point) - \op(\pointalt)}{\point-\pointalt} \geq 0  \text{ for all } \point,\pointalt\in\R^d.
    \label{eq:mere_monotone}
\end{equation}
\end{assumption}
\begin{theorem}\label{thm:minmax}
    Let \cref{asm:solutionset,asm:lipschitz-lineargrowth,asm:wqsm,asm:oracle,asm:uniformbound,asm:monotone} hold. 
    Then the iterates of \eqref{eq:SGDA}, \eqref{eq:SEG}, when run with the step-sizes given in \cref{thm:convergence}, admit a stationary distribution $\distr{SGDA,SEG}{\step}$ such that 
    \begin{equation}
        \ex_{\point\sim\distr{SGDA,SEG}{\step}}\bracks{\gap_\op(\point)}\leq \const\steps{SGDA,SEG},
        \label{eq:game_gap_bound}
    \end{equation}
    where $\gap_\op(\point)$ is
    the restricted merit function
    $\gap_\op(\point):=\sup_{\point^*\in\sols}\inner{\op(\point)}{\point - \point^*}$ and    
    $\const\in\R$ is a constant  and depends on the parameters of the problem. 
\end{theorem}
For the particular case of convex-concave min-max games, the standard notion of \emph{duality gap}, also known as 
\emph{primal-dual optimality gap} or \emph{Nash gap} 
defined as $\dualitygap_f(\theta,\phi) = \max_{\phi'\in\R^{d_2}} f(\theta,\phi') -\min_{\theta'\in\R^{d_1}} f(\theta',\phi)$, is upper bounded by the aforementioned $\gap_\op(\point)$. Here, $\point=(\theta,\phi)$, $f:\R^{d_1}\times\R^{d_2}\to\R$ is a convex function with respect to the first argument and concave with respect to the second one, and $\op=(\nabla_\theta f,-\nabla_\phi f)$ as in \cref{ex:min-max}.

Consequently, let $\text{val}^* = \min_{\theta\in\R^{d_1}}\max_{\phi\in\R^{d_2}}f(\theta,\phi)$ denote the value of this convex-concave game. Then, for the unique stationary distribution $\distr{SGDA,SEG}{\step}$ of the iterates of \eqref{eq:SGDA} and \eqref{eq:SEG}, we have
\begin{equation}
\abs{\ex_{(\theta,\phi)\sim\distr{SGDA,SEG}{\step}}\bracks{f(\theta,\phi)} - \text{val}^*}
\leq \const\steps{SGDA,SEG}.
\label{eq:game_value_bound}
\end{equation}

From \eqref{eq:game_gap_bound} and \eqref{eq:game_value_bound}, we see that in this class of monotone games, \eqref{eq:SGDA} and \eqref{eq:SEG} converge to $\text{val}^*$ \textendash the unique value of the corresponding game at a Nash Equilibrium \textendash within an expected error that is proportional to the stepsize $\steps{SGDA,SEG}$, where the error is measured by the duality gap or the difference in the game value.

\subsection{Bias Refinement in Quasi-Monotone Operators} 
\label{sec:bias-refinement}

Here we focus on the case of quasi-monotone operators (\ie $\wqsmshift = 0$ in \cref{asm:wqsm}), which encompasses a variety of non-monotone and non-convex optimization problems. In this regime, we provide a refined analysis of the stationary distribution induced by \eqref{eq:SGDA} under some smoothness assumptions for the operator and the nature of noise. Specifically, we provide an explicit expansion of the steady-state expectation in terms of the stepsize, which allows us to employ the Richardson-Romberg (RR) bias refinement scheme \cite{gautschi2011numerical} to construct a new estimate provably closer to the optimal solution. Our result is a strict generalization of \cite{Dieuleveut20-bach-SGD}, which requires co-coersive noisy first-order oracles.

\begin{assumption} \label{asm:fourth_noise}
The operator $\op$ is $\lips$-Lipschitz and $C^4(\R^d)$-smooth (\ie $\sup_{\point\in\R^d}\norm{\grad^i \op(\point)}<\infty$ for all $i=1,\hdots, 4$). Furthermore, the noise has bounded kyrtosis, meaning that $\exof{\norm{\noise_\run(\point)}^4} < \kyrt^4 $ for all $\point \in \R^d$ with   
the covariance tensor $\point \mapsto \mathcal{C}(\point) :=  \ex\bracks{\noise_\run(\point)^{\otimes 2}}$ being 3 times smoothly differentiable, meaning $\norm{\mathcal{C}^{(i)}(\point)} < G,\forall \point$, 
for $i\in\braces{1,2,3}$.
\end{assumption}

\begin{theorem}\label{thm:bias_expansion}
    Suppose \cref{asm:solutionset,asm:lipschitz-lineargrowth,asm:wqsm,asm:oracle,asm:uniformbound,asm:fourth_noise} hold. There exists a threshold $\thres$ such that if  $\step\in(0,\thres),$ then \eqref{eq:SGDA} admits a unique stationary distribution $\pi_\step$ 
    and 
    \begin{equation} 
    \label{eq:bias_expansion}
        \ex_{\point\sim\pi_\step}\bracks{\point} -\state^* = \step\Delta(\state^*) + \bigoh(\step^2),
    \end{equation}
    where $\Delta(\state^*)$ is a vector independent of the choice of step-size $\step$.
\end{theorem}

Note that \cref{eq:bias_expansion} is an equality (up to a second order term). In the setting of \cref{thm:bias_expansion}, this equality gives a more precise characterization of the bias than the upper bound \eqref{eq:bias_upper_bound} applied to $\tf(\point) = \point.$

An immediate implication of \cref{thm:bias_expansion} is that one can use the following RR refinement scheme to obtain a better estimate of $\state^*.$ Consider running two \eqref{eq:SGDA} recursions with step-size $\step$ and $2\step$ and denote the corresponding averaged iterates
by $(\bar{\state}^{\step}_\run)_{\run\geq 0}$ and $(\bar{\state}^{2\step}_\run)_{\run\geq 0}$, respectively.
Let us denote by $\pi_{\step}$ and $\pi_{2\step}$ the resulting unique stationary distributions. By our result on LLN (cf.\ \cref{cor:LLN}), the averaged iterates $(\bar{\state}^{\step}_\run)_{\run\geq 0}$ and $(\bar{\state}^{2\step}_\run)_{\run\geq 0}$ converges to $\ex_{\point\sim\pi_{\step}}\bracks{\point}$ and $\ex_{y\sim\pi_{2\step}}\bracks{y}$, respectively. 
Note that \cref{eq:bias_expansion} implies that 
$$\big(\ex_{\point\sim\pi_{\step}}\bracks{2\point}-\ex_{y\sim\pi_{2\step}}\bracks{y} \big)-\state^* = \bigoh(\step^2).$$
Therefore, the RR refinement of the averaged iterates, $(2\bar{\state}^{\step}_\run-\bar{\state}^{2\step}_\run)_{\run\geq 0}$, converge to a limit that is closer to the optimal solution $\state^*$ by a factor of $\step.$

\subsection{Experiments}
\label{sec:experiments}

We conduct a series of experiments to empirically observe and validate our results. We focus on strongly convex-concave games with two players, for which we have adapted the code of the repository of~\cite{HIMM20}.
In particular, for the first two sets of experiments (\cref{fig:bias,fig:normal-different-iterations,fig:normal-steps}), we consider a strongly convex-concave min-max game, $\min_{x_1\in\R^d} \max_{x_2 \in \R^d} f(x_1,x_2)$, with $f:\R^d\times\R^d\to\R$ given by
\begin{align*}
    f(x_1, x_2) = x_1^\top A_1 x_1 - x_2^\top A_2 x_2 + (x_1^\top B_1 x_1)^2 - (x_2^\top B_2 x_2)^2 + x_1^\top  C x_2,
\end{align*}
where $d=50$, each of $A_1,A_2,B_1,B_2 \in \R^{d\times d}$ is a random positive definite matrix, and $C$ is a random matrix. Note that the global solution of the game is $x^* = (x_1^*, x_2^*) = (0,0)$ with value $f(x^*_1, x_2^*) = 0$. The operator associated with the above game is 
$$\op(x) = \op((x_1,x_2)) = (\nabla_{x_1} f(x_1, x_2), -\nabla_{x_2}f(x_1,x_2)).$$ 
The stochastic oracle outputs $\op(x) + Z$, where $Z \sim \mathcal{N}(0,\sigma^2 I)$ is  Gaussian noise with $\sigma=0.5.$

\begin{figure}[t]
\centering
\begin{subfigure}[b]{.38\linewidth}
\centering
\includegraphics[width=\textwidth]{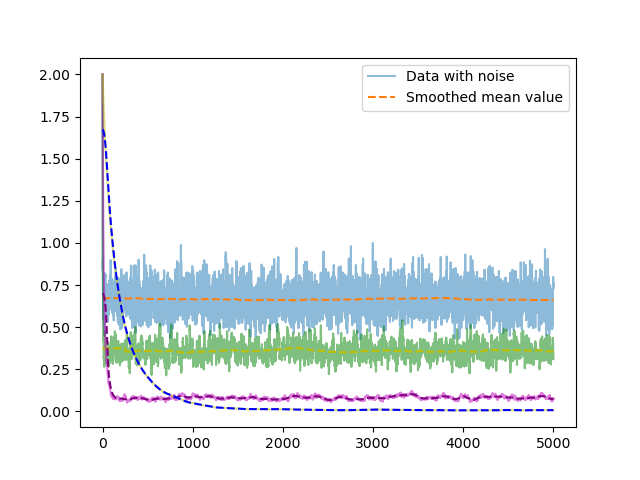}
\caption{SGDA.}
\label{fig:GD_bias}
\end{subfigure}%
\begin{subfigure}[b]{.38\linewidth}
\centering
\includegraphics[width=\textwidth]{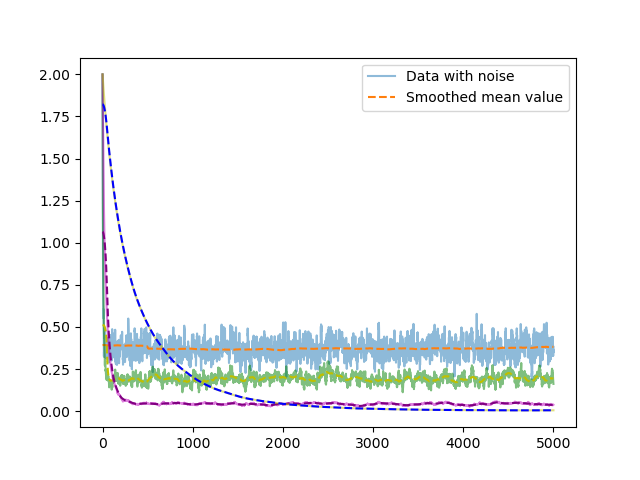}%
\caption{SEG.}
\label{fig:EG_bias}
\end{subfigure}
%
\caption{Convergence and squared error under different step-sizes for SGDA and SEG.}
\label{fig:bias}
\end{figure}

We started by plotting in \cref{fig:GD_bias,fig:EG_bias} the squared error $\norm{\state_\run -\sol}^2$ for \eqref{eq:SGDA} and \eqref{eq:SEG} for step-sizes $\gamma \in \{0.1 , 0.05 , 0.01 , 0.001\}$, corresponding to the four curves from top to bottom; the parameter $\stepsalt{SEG}$ for \eqref{eq:SEG} is set to $0.5$.  We observe  a decay of the steady-state error as a function of the step-size. In fact, the decay is almost linear for both algorithms, which is consistent with our theoretical bound~\eqref{eq:bias_upper_bound} applied to the test function $\tf(\point) = \norm{\point - \point^*}$.  

\begin{figure}[h]
\centering
\begin{subfigure}[b]{.45\linewidth}
\centering
\includegraphics[width=0.8\textwidth]{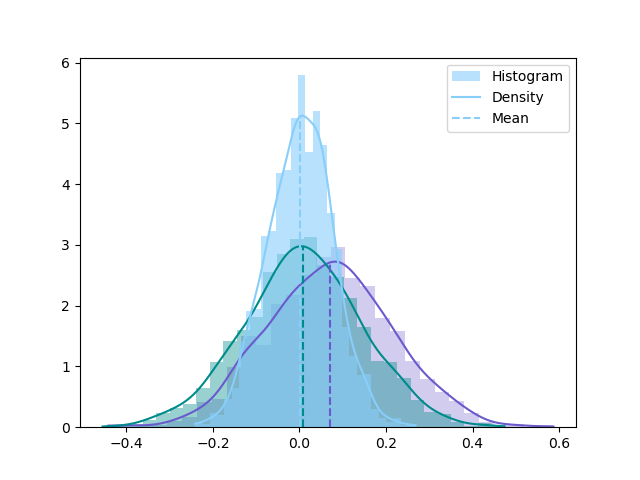}%
\caption{SGDA}
\label{fig:SGDA-diff}
\end{subfigure}%
\begin{subfigure}[b]{.45\linewidth}
\centering
\includegraphics[width=0.8\textwidth]{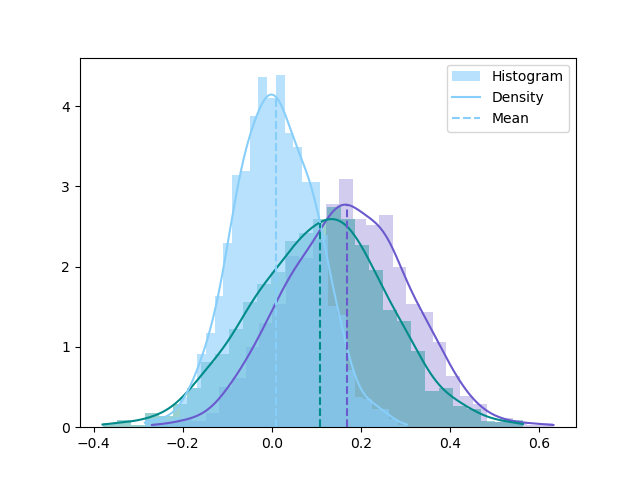}%
\caption{SEG}
\label{fig:SEG-diff}
\end{subfigure}%
\caption{Results for $100$  (light purple), $200$ (light green), $1000$ (light blue) iterations (or from right to left).
}
\label{fig:normal-different-iterations}
\end{figure}

\begin{figure}[t]
\centering
\begin{subfigure}[b]{.45\linewidth}
\centering
\includegraphics[width=0.8\textwidth]{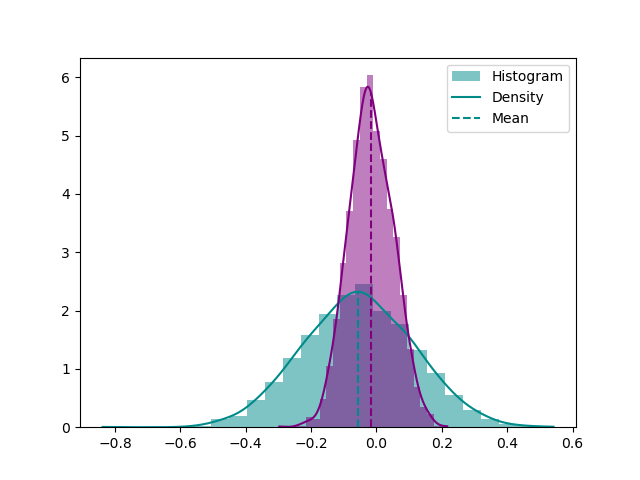}%
\caption{SGDA}
\label{fig:SGDA-steps}
\end{subfigure}%
\begin{subfigure}[b]{.45\linewidth}
\centering
\includegraphics[width=0.8\textwidth]{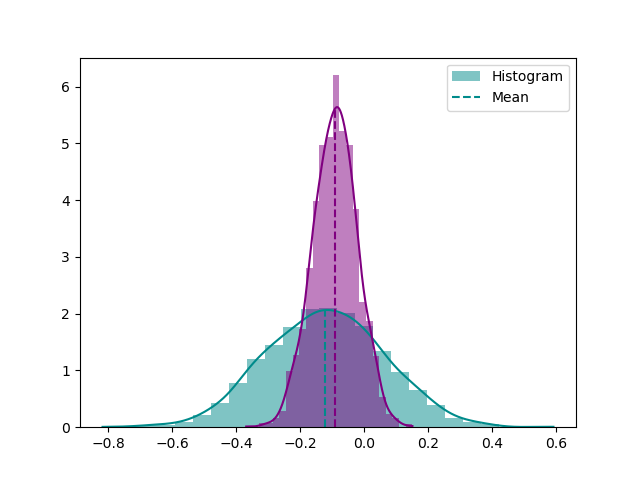}%
\caption{SEG}
\label{fig:SEG-steps}
\end{subfigure}
\caption{Histograms for two different step-sizes. Green: $\gamma = 0.1$. Purple: $\gamma = 0.001$.}
\label{fig:normal-steps}
\end{figure}
The second set of experiments examines the central limit theorem (CLT). We use as a test function the value of the game $f(\state_\run)$ evaluate at the iterate, and we observe the behavior of its averaged evaluations after  $100,200$ and $1000$ iterations. To do so we run both algorithms with step-size $\step = 0.005$ for the aforementioned number of iterations and keep the sum of the evaluations, normalized with $\sqrt{\text{iterations}}$. We repeat this experiment 2000 times and report the histograms in \cref{fig:normal-different-iterations}. We observe how the distributions are concentrated closer to the actual value of the game (which is zero) as the number of iterations is increased. In \cref{fig:normal-steps} we run both algorithms in the previous setting for 1000 iterations and two different step-sizes $0.1$ and $0.001$. We observe how the histogram is concentrated closer to the actual value of the game for smaller step-size.


Lastly, to investigate the effect of the RR refinement scheme, we perform an  experiment on a slightly more complicated game. Define the scalar function $h(z):=\log(1+e^{z})$, which is convex. Consider a strongly convex-concave min-max game with $f:\R\times\R\to\R$ given by
\begin{align*}
    f(x_1, x_2) =  h(x_1) + h(-2x_1) - h(x_2) - h(-2x_2) + 0.1x_1^2 - 0.1x_2^2 + 0.1 x_1 x_2.
\end{align*}
The operator $\op$ and the stochastic oracle are defined in the same way as before. The global solution of this game is $x^* = (x_1^*, x_2^*) \approx (0.3268,0.3801)$.

We run the \eqref{eq:SGDA} algorithm with two different step-sizes $\step$ and $2\step$, where $\step=0.1$. In \cref{fig:extrapolation_new}, we plot the error $\|\bar{x}_t - x^*\|^2$ of the averaged iterate $\bar{x}_t:=\frac{1}{t}\sum_{i=1}^t x_i$ with the two stepsizes,  as well as that of the RR refinement scheme (cf.\ Section~\ref{sec:bias-refinement}). The error achieved by the RR refinement is an order of magnitude better than vanilla \eqref{eq:SGDA}. This is consistent with the bias reduction effect predicted by our theoretical result in Section~\ref{sec:bias-refinement}.

\begin{figure}[h]
\centering
\includegraphics[width=0.6\textwidth]{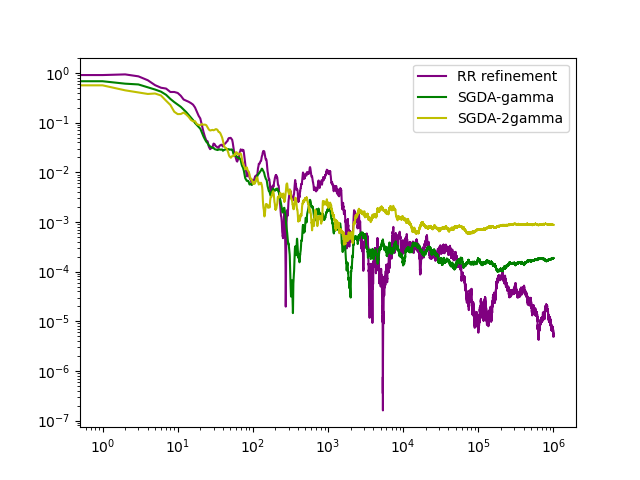}
\caption{Errors of the average iterates of SGDA and RR refinement.}
\label{fig:extrapolation_new}
\end{figure}

\section{Concluding remarks}
\label{sec:discussion}
In this work, we delve into the probabilistic structures inherent in Stochastic Extragradient and Stochastic Gradient Descent Ascent algorithms, widely used in min-max optimization and variational inequalities problems. By treating constant step-size variants of SEG/SGDA as time-homogeneous Markov Chains, we establish a Law of Large Numbers and a Central Limit Theorem, revealing the existence of a unique invariant distribution and the asymptotic normality of the averaged iterate. For a wide class of convex-concave games, we characterize the intrinsic bias of these methods w.r.t.\ the game's value. Lastly, we demonstrate that the Richardson-Romberg refinement scheme enhances the proximity of the averaged iterate to the global solution for quasi-monotone variational inequalities.

As a result of this study, several intriguing open questions arise. The extension of Markovian analysis to broader operator families, and their potential applications in statistical inference, adversarial training, and robust machine learning present exciting research opportunities. Investigating how the methods used in this study can be applied to other established optimization algorithms, such as Optimistic Gradient Descent Ascent, which requires higher-order Markov process analysis, is another promising line of research. Exploring different geometries and studying robustness in reinforcement learning also offer interesting prospects.

\section*{Acknowledgments}
\begingroup
%
%
Emmanouil V. Vlatakis-Gkaragkounis is grateful for financial support by the  Post-Doctoral FODSI-Simons Fellowship. Yudong Chen is partially supported by NSF grants CCF-1704828 and CCF-2233152.
Qiaomin Xie is supported in part by NSF grant CNS-1955997.
This project started when EVVG, YC and QX were attending the Data-Driven Decision Processes program of the Simons Institute for the Theory of Computing.
\ackperiod
\endgroup

\bibliographystyle{icml}
\bibliography{./bibtex/IEEEabrv,bibtex/Bibliography-PM,./bibtex/newpapers,./bibtex/VI,./bibtex/SA}

\newpage
\appendix
\numberwithin{equation}{section}		
\numberwithin{lemma}{section}		
\numberwithin{proposition}{section}		
\numberwithin{theorem}{section}		
\numberwithin{corollary}{section}		

\renewcommand{\contentsname}{Organization of the appendix}
\addtocontents{toc}{\protect\setcounter{tocdepth}{3}}
\tableofcontents

\clearpage

\section{Background in Continuous-Space Markov Chains}\label{app:markov}

In this preliminary segment, we furnish the basic concepts and tools for studying Markov chains defined on a continuous state space. These results subsequently form the foundational basis for the theorems we establish regarding our algorithms.

\subsection{Basic Setup}
To explain various concepts for a Markov chain, we first set up our space and identify the events of interest.  This process is grounded in the conventional framework of a $\sigma$-algebra, which facilitates the comprehension of these events. Formally, we denote the (sub)-$\sigma$-algebra of $\filter$ of events up to the $\time$-th iteration with $\filter_\time$ (including the $\time$-th iteration). We denote by $\borel(C)$ the $\sigma$-algebra of Borel sets of $C$. We also denote the Markov kernel (Generalized Transition Matrix) on $\R^d$, $\borel(\R^d)$ associated either with \eqref{eq:SGDA} or \eqref{eq:SEG}  to be\footnote{It would be clear from the context in which algorithm we refer to. If not we will specify it using subscripts.}
\begin{equation}
    P(x,S) =\probof{\state_{\time+1}\in S|\state_\time =x} \text{ almost surely }\forall S\in\borel(\R^d),\forall x\in\R^d, \forall t\in\N.
\end{equation}
We also define the $m$-th power of the kernel iteratively: $P^1(x,S):= P(x,S)$ and for $m>1$, we define
\begin{equation}
    P^{m+1}(x,S) = \int_{x'\in \R^d}P(x,\dd x')P^m(x',S) \text{ for all }x\in\R^d \text{ and }S\in\borel(\R^d).
\end{equation}
Additionally, for any function $\tf:\R^d\to\R$ and any $m\geq 1$, we define $P^m\tf: \R^d\rightarrow \R$ as
\begin{equation}
    P^m\tf(x) =\int_{x'\in \R^d}\tf(x')P^m(x,\dd x') \text{ for all }x\in \R^d.
\end{equation}
\begin{definition}[Time-homogeneous]\label{def:hom}
A stochastic process $\Phi=(\Phi_t)_{t=0}^{\infty}$ is called a time-homogeneous Markov chain with transition probability kernel $P(x,A)$ and initial distribution $\mu$ if the finite dimensional distributions of $\Phi$ satisfy
\begin{equation}
    P_{\mu}\parens{\Phi_0\in A_0, \Phi_1\in A_1 ,\hdots \Phi_n\in A_n} = \int_{y_0\in A_0}\dots\int_{y_{n-1}\in A_{n-1}}\mu(\dd y_0)P(y_0,\dd y_1) \cdots P(y_{n-1},A_n)
\end{equation}
for any $n$ and all $A_i\in \borel(\R^d).$
\end{definition}
\subsection{Irreducibility, Recurrence, and Aperiodicity}
\para{\emph{Irreducibility}}
\begin{definition}[$\irr-$irreducible]\label{def:irr}  A Markov chain is $\varphi$-irreducible if  there exists a measure $\varphi$ on $\borel(\R^d)$ such that for all $x\in \R^d$ whenever $\varphi(A) >0$, there exists $n>0$, possible depending on $x,A$ such that  that $P^n(x,A) >0$. Per convention, we always take $\varphi$ to be a ``maximal'' irreducibility measure, denoted by $\irr$, and say that the chain is $\irr-$irreducible. 
\end{definition}

For this definition we combine Proposition 4.2.1 and Proposition 4.2.2 from \cite{Meyn12_book}. Consider a $\irr-$irreducible Markov chain, we use $\borel^+(\R^d)$ to denote the set of sets $A\in \borel(\R^d)$ such that $\varphi(A)>0$.

\para{\emph{Recurrence}}
\begin{definition}[Recurrent]\label{def:recurrent} Consider a Markov chain $\Phi=(\Phi_t)_{t=0}^\infty$ with transition kernel $P$. Let $\eta_A: = \sum_{\time =\start}^\infty \one\braces{\Phi_\time\in A}$ for some set $A$. Assume that $\Phi$ is $\irr$-irreducible, then we say that 
\begin{itemize}
    \item \textbf{(null)-Recurrent:} The set $A$ is called recurrent if $\exof{\eta_A\given \Phi_0 = x} = \infty$ for all $x\in A$. If every set in $\borel^+(\R^d)$ is recurrent then we call $\Phi$ recurrent.
    \item \textbf{Positive recurrent:} The set $A$ is called positive if $\lim \sup_{n\to\infty}P^{n}(x,A) > 0$ for all $x\in A$. If every set $A\in \borel^+(\R^d)$ is positive then $\Phi$ is called positive recurrent.
    \item \textbf{Harris recurrent:} The set $A$ is called Harris recurrent if $\probof{\eta_A = \infty\given\Phi_0 = x}=1$ for all $x\in A$.  If every set $A\in\borel^+(\R^d)$ is Harris recurrent, then $\Phi$ is called Harris recurrent.
\end{itemize}
\end{definition}
\para{\emph{Aperiodicity}}

\begin{definition}[Strongly Aperiodic]\label{def:aperiodic}
An irreducible chain is called strongly aperiodic if  there exists a set $A$, such that there exists a non-trivial measure $\nu_1$ on $\borel(\R^d)$ satisfying $\nu_1(A)>0$, and for all $x\in A$ and $S\in\borel(\R^d)$,
\begin{equation}
    P(x,S)\geq \nu_1(S).
\end{equation} 
\end{definition}
Looking at the bigger picture and drawing insight from traditional discrete space Markov chains, if we make a selection such that $S\gets A$, then we achieve $P(x,A) \geq \nu_1(A)>0$. This suggests that the set $A$ is associated with a self-loop, as it has a positive probability of returning to itself.

\subsection{Small Sets, Petite Sets, and Minorization Condition}

We next introduce several concepts that pave the way for systematically and efficiently establishing the convergence rate of a Markov
chain, other than in an ad-hoc manner. 

We first introduce the Minorization Condition. Using this condition is similar in a way as thinking
about coupling.

\begin{definition}[Minorization Condition]\label{def:minor}
For some $\delta >0$, some $\set\in\borel(X)$ and some probability measure $\nu$ with $\nu(\set^c) = 0$ and $\nu(C) = 1$: 
\begin{equation}
    P(x,A)\geq \delta\one_{\set}(x)\nu(A) \text{ for all }A\in\borel(\R^d), x\in \R^d.
\end{equation}
\end{definition}
If $C$ was the entire $\R^d$, the condition requires every state in the state space to be within reach of any other state. We could then minorize the transition probability with a density $\nu$ scaled by a parameter $\delta$. This
is equivalent to finding a sliver of a probability distribution where all the transition probabilities
“overlap” with each other; see Figure~\ref{fig:my_label} for an illustration. However, in continuous spaces having $C=\R^d$ is usually impossible. The set where such a condition holds is called ``small''. 

\begin{definition}[Small Sets]\label{def:small}
A set $C\in\borel(\R^d)$ is called a small set if there exists an $m \in \N_+$ and a non-trivial measure $\nu_m$ on $\borel(\R^d)$ such that for all $x\in C$, $B\in\borel(\R^d)$,
\begin{equation}
    P^m(x,B)\geq \nu_m(B)
\end{equation}
The set $C$ is called $\nu_m$-small.
\end{definition}
Let $a=\braces{a(n)}$ be a distribution or probability measure on $\N_{+}$ and consider the associated Markov chain $\Phi_a$ with probability transition kernel   
$$K_a :=\sum_{n=0}^\infty P^n(x,A)a(n) \; x\in\R^d,A\in\borel(\R^d).$$
$\Phi_a$ is called the $K_a$-chain with sampling distribution $a$. We can interpret $\Phi_a$ as the chain $\Phi$ sampled in points according to the distribution $a$. When $a=\delta_m$ is the Dirac measure with $\delta_m(m) = 1$, then the $K_{\delta_m}$-chain is called the $m$-skeleton with transitional kernel $P^m$. With this at hand we define below the petite sets.
\begin{definition}[Petite Sets] We will call a set $C\in\borel(\R^d)$ $\nu_a$-petite if the sampled chain satisfies the bound
\begin{equation}
    K_a(x,B) \geq \nu_a(B)
\end{equation}
    for all $x\in C$, $B\in\borel(\R^d)$, where $\nu_a$ is a non-trivial measure on $\borel(\R^d)$.
\end{definition}
\begin{proposition}[Proposition 5.5.3 in \cite{Meyn12_book}]\label{prop:petite}
    If a set $C\in\borel(\R^d)$ is $\nu_m$-small then it is $\nu_{\delta_m}$-petite for some $\delta_m>0$.
\end{proposition}

\subsection{Foster-Lyapunov Arguments}
Given that only small sets can be found in our setting, in order to prove geometric convergence to a unique stationary distribution we will leverage the  generalized version of Foster-Lyapunov condition, dubbed as (V4) in the cited book \cite{Meyn12_book}.

The following theorem gives a sufficient criterion for the positive recurrence and existence of an invariant distribution of a Markov chain in terms of a Lyapunov function $V$. Intuitively, the value $V(x)$ for any state $x$ attained by Markov chain denotes “energy”
or “potential” of that state. The idea is that if the mean energy decreases for all but some small set,
the Markov chain keeps returning to level-sets close to minimum of the energy. That is, the Markov chain is positive recurrent.

\begin{definition}[Geometric Drift Property]\label{def:drift}
    There exists an extended-real valued function $f:\R^d\to[1,\infty]$, a measurable set $C$, and constants $\beta >0$, $b<\infty$ such that 
    \begin{equation}
        \Delta f(x) \leq -\beta f(x) + b\one_C(x), x\in \R^d,
    \end{equation}
    where $\Delta f(x) = \int_{y\in \R^d} P(z,dy)f(y) - f(x).$
\end{definition}

\begin{figure}[h!]
    \centering
    \includegraphics[scale=0.8]{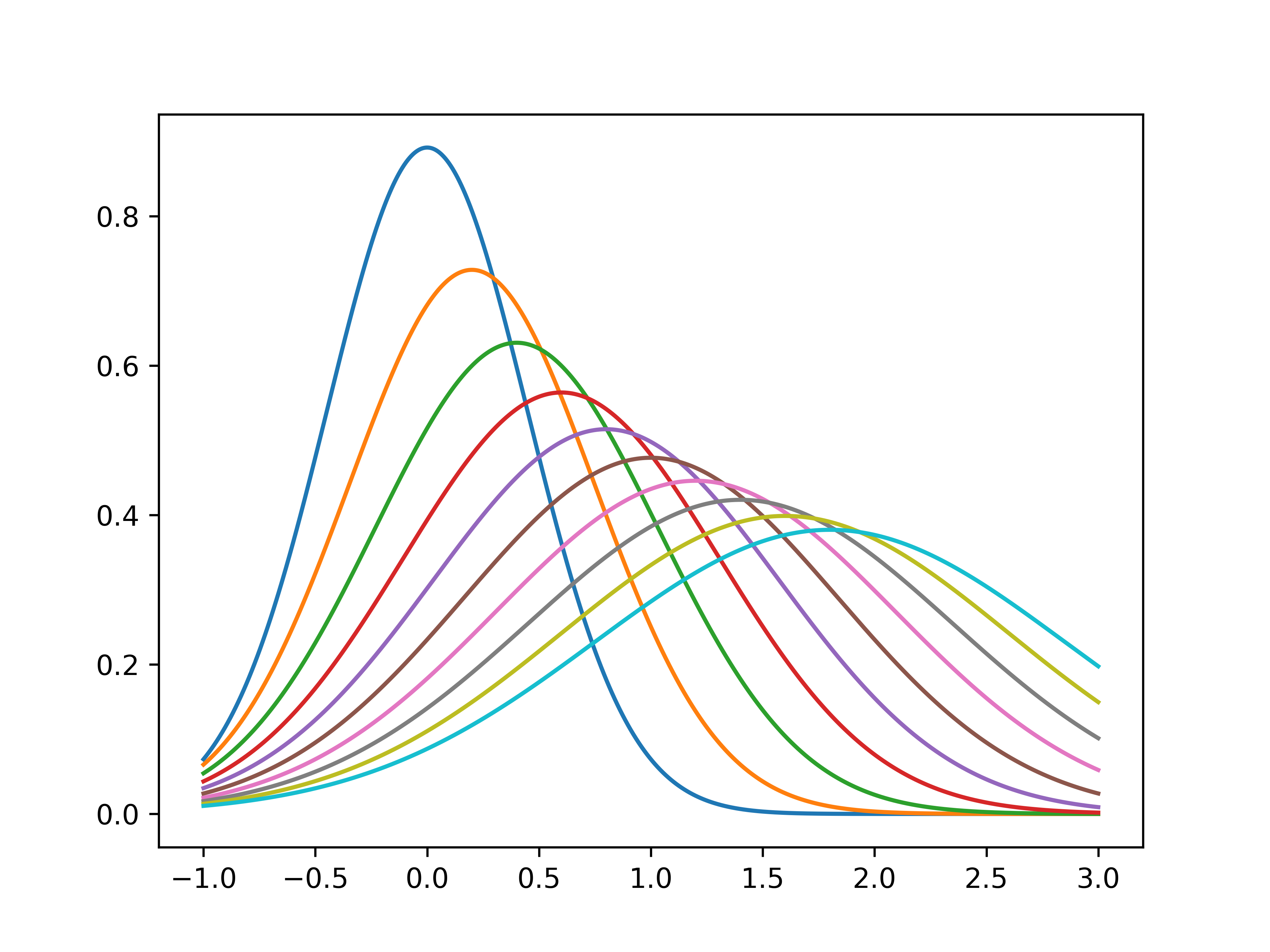}
    \caption{Example of transition kernel $P(x,C)$ for $x\in\R^d$}
    \label{fig:my_label}
\end{figure}
\clearpage
\section{Omitted Proofs of Section~\ref{sec:algorithms}}
\label{app:proofs}

\subsection{\eqref{eq:SGDA} and \eqref{eq:SEG} are time-homogeneous  Markov chains in $\R^d$}
\begin{lemma}
Given a constant step-size, the stochastic gradient descent ascent and stochastic extra-gradient as described by Equation \eqref{eq:SGDA} and \eqref{eq:SEG} can be equivalently modeled as a time-homogeneous continuous Markov chain in $\R^d$.
\end{lemma}
\begin{proof}
We start with \eqref{eq:SGDA} simple case:
\[\tag{SGDA}
   \state_{\run+1} = \state_\run -\steps{SGDA}\op_\run= \state_\run - \steps{SGDA}(\op(\state_\run)+\noise_\run(\state_\run)).
\]
By this definition we get that
    \begin{align*}
        P(x,B) &= \prob\parens*{\state_{\time+1}\in B|\state_\time=\point}\\
        &=\prob\parens*{\state_\time -\step(\op(\state_\time) +\noise_\time(\state_\time))\in B|\state_\time = \point}     \\
        &=\prob\parens*{\point - \step(\op(\point) +\noise_\time(\point))\in B}\\
        &=\prob\parens*{\noise(\point)\in (\dfrac{\point}{\step}-\op(\point))+(-\dfrac{1}{\step}B)},
    \end{align*}
    where $\left(\noise_\time(\point) \right)_{t\geq 0}\sim \noise(\point)$, since we assume i.i.d noise random fields. Hence, $P(x,B)$ is shown to be independent of both time $t$ and preceding iterations, given the current state. This affirms that the stochastic gradient descent model described by Equation \eqref{eq:SGDA} indeed exhibits the property of a time-homogeneity, substantiating its classification as a Markov chain.

For the case of \eqref{eq:SEG}, an equivalent form which will come at hand throughout our analysis is given below
\begin{equation}\label{eq:SEGalt}
\begin{aligned}
   \state_{\time+1} &= \state_{\time} - \stepsalt{SEG}\steps{SEG}\op_{\time+1/2}\\
   &=\state_\time -\stepsalt{SEG}\steps{SEG}\op_{\time+1/2}\\
   &=\state_\time -\stepsalt{SEG}\steps{SEG}\op_{\time+1/2}\\
   &= \state_\time  - \stepsalt{SEG}\steps{SEG}(\op(\state_{\time+1/2}) +\noise_{\time+1/2}(\state_{\time+1/2}))\\
   &=\state_\time -\stepsalt{SEG}\steps{SEG}\op(\state_\time -\steps{SEG}(\op(\state_\time) +\noise_\time(\state_\time)))\\
   &\quad- \stepsalt{SEG}\steps{SEG}\noise_{\time+1/2}(\state_\time -\steps{SEG}(\op(\state_\time) +\noise_\time(\state_\time))).
\end{aligned}
\end{equation}
Thus for the transition kernel we get that
\begin{align*}
        P(x,B) =& \prob\parens{\state_{\run+1}\in B|\state_\run = x}\\
        =&\prob\parens{\state_{\run} - \stepalt\step\op(\state_\run -\step\op(\state_\run)
        -\step\noise_\run(\state_\run))\\
        &-\stepalt\step\noise_{\run+1/2}(\state_\time -\step\op(\state_\time) -\step\noise_\time(\state_\time))\in B|\state_\time =x}\\
        =& \prob\parens{\point - \stepalt\step\op(\point -\step\op(\point) -\step\noise_\time(\point))\\
        &-\stepalt\step\noise_{\time+1/2}(\point-\step \op(\point)-\step\noise_\time(\point))\in B},
    \end{align*}
    where $\noise_\time(\point)\sim \law(\noise^A(\point))$, $\noise_{\time+1/2}(\point)\sim\law(\noise^{B}(\point))$ and $\noise^A(\point)\perp\noise^B(\point)$, identically distributed. Thus,
    \[
        P(x,B) = \int_{\snoise\in\R^d} \pdf_{\noise^A(\point)}\parens*{\snoise} \prob\parens*{\point -\stepalt\step\op(\point-\step\op(\point)-\step\snoise) -\stepalt\step\noise^B(\point-\step\op(\point) -\step\snoise)\in B}\dd\snoise.
    \]
    So again, $P(x,B)$ is shown to be independent of both time $t$ and preceding iterations, given the current state. This affirms that the stochastic gradient descent model described by Equation \eqref{eq:SGDA} indeed exhibits the property of a time-homogeneity, substantiating its classification as a Markov chain, completing the proof for the case of \eqref{eq:SEG}.
\end{proof}

\subsection{Geometric convergence up to constant factor}

\begin{fact}\label{fact:norm} Let $a,b,c\in\R^d$, then the following holds
\begin{equation}
    \norm{a+b+c}^2 \leq 3(\norm{a}^2 + \norm{b}^2 +\norm{c}^2).
\end{equation}
\end{fact}

We split \cref{thm:convergence} into two different lemmas for each of the algorithms. We start by presenting \cref{eq:SGDA}.

\begin{lemma}\label{lem:exponentialSGDA} Suppose that \cref{asm:solutionset,asm:lipschitz-lineargrowth,asm:wqsm,asm:oracle} hold then the  iterations $(\state_\time)_{\time\geq \start}$ of  \eqref{eq:SGDA}, if the step-size is $\step<\frac{\wqsmscale}{\growth^2},$ satisfy:
\begin{equation*}
    \exof{\norm{\state_{\time +1} -\state^*}^2\given\filter_{\time}} \leq (1-\const)^\time\norm{\state_\start-\state^*}^2 +\constalt
\end{equation*} 
for some constants $\const\in (0,1)$ and $\constalt\in(0,+\infty)$ that depend on the choice of step-size, as well as the parameters of the problem.
\end{lemma}

\begin{proof}
For simplicity, we drop the exponent SGDA of the step-size and we  write $\step$ for the constant step-size used while the algorithm is run. We  now start by writing
\begin{equation}\label{eq:aux}
\begin{aligned}
    \norm{\state_{\time +1} -\state^*}^2 &= \norm{\state_{\time} - \step\op_\time -\state^*}^2\\
    &=\norm{\state_{\time}-\state^*}^2 -2\step\inner{\op_\time}{\state_{\time}-\state^*} + \step^2\norm{\op_\time}^2.
 \end{aligned}
\end{equation}
 By taking the expectation condition on the filtration $\filter_\time$, we have that
 \begin{equation}\label{eq:intSGDA}
 \begin{aligned}
     \exof{\norm{\state_{\time +1} -\state^*}^2\given\filter_\time} &= \norm{\state_{\time}-\state^*}^2 - 2\step\inner{\op(\state_\time)}{\state_\time -\state^*} + \step^2\exof{\norm{\op_\time}^2\given\filter_\time}\\
     &= \norm{\state_{\time}-\state^*}^2 - 2\step\inner{\op(\state_\time)}{\state_\time -\state^*} + \step^2\exof{\norm{\op(\state_\time)}^2}\\ &\quad+\step^2\exof{\norm{\noise_\time(\state_\time)}^2\given\filter_\time},
 \end{aligned}
 \end{equation}
 since $\state_\time$ is $\filter_\time-$measurable  and $\exof{\op_\time\given\filter_\time} = \op(\state_\time)$. By \cref{asm:oracle} we have that 
\begin{equation}\label{eq:SGDA1}    \exof{\norm{\noise_\time(\state_\time)}^2\given\filter_\time}\leq \sbound^2,
\end{equation} while  \cref{asm:wqsm} implies that
 \begin{equation}\label{eq:SGDA2} 
    - 2\step\inner{\op(\state_\time)}{\state_\time -\state^*} \leq -2\wqsmscale\step\norm{\state_\time -\state^*}^2 + 2\wqsmshift\step.
 \end{equation}
Finally, using the assumption that the operator has at most linear growth (\cref{asm:lipschitz-lineargrowth}) we have that for all $\point\in\R^d$,
\begin{align} 
    \norm{\op(\point)} &\leq \growth(1+\norm{x}) \leq \growth(1+\norm{\sol} +\norm{\point-\sol}) \Rightarrow \nonumber \\
    \norm{\op(\point)}^2&\leq \growth^2(1+\solb +\norm{\point-\sol})^2 \nonumber \\
    &\leq 2\growth^2((1+\solb)^2 +\norm{\point-\sol}^2 ). \label{eq:SGDA3}
\end{align}
By substituting \cref{eq:SGDA1,eq:SGDA2,eq:SGDA3} to \cref{eq:intSGDA},  we get that
\begin{equation} \label{eq:SGDA_one_step}
  \exof{\norm{\state_{\time +1} -\state^*}^2\given\filter_\time} \leq (1-2\wqsmscale\step+2\step^2\growth^2)\norm{\state_\time -\state^*}^2  +(2\wqsmshift\step + 2\step^2\growth^2(1+\solb^2) + \step^2\sbound^2).
\end{equation}
Now if 
\begin{align*}
& 1-2\wqsmscale\step+2\step^2\growth^2 <1\\
\Leftrightarrow \quad & 2\step^2\growth^2 < 2\wqsmscale\step\\
\Leftrightarrow \quad  &   0< \step < \dfrac{\wqsmscale}{\growth^2},
\end{align*}
and by letting $1-\const = 1-2\wqsmscale\step+2\step^2\growth^2 < 1$ and $\constalt = \frac{2\wqsmshift\step + 2\step^2\growth^2(1+\solb^2) + \step^2\sbound^2}{\const}$, we can rewrite \cref{eq:SGDA_one_step} as
\begin{equation*}
    \exof{\norm{\state_{\time +1} -\state^*}^2\given\filter_\time} \leq (1- \const)\norm{\state_\start -\state^*}^2 + c\constalt.
\end{equation*}
Therefore, we have
\begin{equation*}
    \exof{\norm{\state_{\time +1} -\state^*}^2} \leq (1- \const)^\time\norm{\state_\start -\state^*}^2 + \constalt \text{ for all }\time\geq \start.
\end{equation*}
\end{proof}

We proceed on proving a similar lemma for the case of \eqref{eq:SEG}. To do so, we first introduce and analyze two intermediate steps.

\begin{proposition}\label{prop:auxiliary}Consider that \eqref{eq:SEG} is run and let $\state^*\in\sols$, $\rvar_\time = \op_{\time+1/2} = \op(\state_{\time+1/2}) +\noise_{\time+1/2}(\state_{\time+1/2})$, where $\op, \noise$ satisfy \cref{asm:lipschitz-lineargrowth,asm:wqsm,asm:oracle}, $\step\in\R$ is a constant step-size.  If $\step\leq \dfrac{1}{\sqrt{3}\lips}$ then 
\begin{equation*} \step^2\exof{\norm{\rvar_\time}^2\given\filter_\time}\leq 2\step\exof{\inner{\rvar_\time}{\state_\time-\state^*} \given\filter_\time} +2(\wqsmshift\step+3\sbound^2\step^2).
\end{equation*}
\end{proposition}

\begin{proof}
    Consider the auxiliary variable $\hat{\state}_{\time+1} = \state_\time - \step\rvar_\time$, then  we have
    \begin{equation*}
        \norm{\hat{\state}_{\time+1}-\state^*}^2 = \norm{\state_\time - \state^*}^2 -2\step\inner{\rvar_\time}{\state_\time-\state^*} +\step^2\norm{\rvar_\time}^2.
    \end{equation*}
By taking the expectation given the filtration $\filter_\time$, we have 
\begin{align}\label{eq:eq2}
    \exof{\norm{\hat{\state}_{\time+1}-\state^*}^2\given\filter_\time} &= \norm{\state_\time - \state^*}^2 - 2\step\exof{\inner{\rvar_\time}{\state_\time-\state^*} \given\filter_\time} + \step^2\exof{\norm{\rvar_\time}^2\given\filter_\time}.
\end{align}
Notice that 
\begin{align*}
    \exof{\inner{\rvar_\time}{\state_\time -\state^*}\given\filter_\time} &= \exof{\inner{\op(\state_{\time+1/2})+\noise_{\time+1/2}(\state_{\time+1/2})}{\state_\time -\state^*}\given\filter_\time}\\
&=\exof{\inner{\op(\state_{\time+1/2})}{\state_\time -\state^*}\given\filter_\time}\\
&=\exof{\inner{\op(\state_{\time}-\step\op_{\time})}{\state_\time -\state^*}\given\filter_\time}
\end{align*}
Thus, \cref{eq:eq2} becomes
\begin{align*}
   \exof{\norm{\hat{\state}_{\time+1}-\state^*}^2\given\filter_\time} =& \norm{\state_\time - \state^*}^2 - 2\step\exof{\inner{\op(\state_{\time}-\step\op_\time)}{\state_\time -\step\op_\time - \state^*}\given\filter_\time}\\
   &-2\step^2\exof{\inner{\op(\state_{\time}-\step\op_\time)}{\op_\time}\given\filter_\time}+ \step^2\exof{\norm{\rvar_\time}^2\given\filter_\time}.
\end{align*}
We can now use \cref{asm:wqsm} and we get
\begin{align*}
    \exof{\norm{\hat{\state}_{\time+1}-\state^*}^2\given\filter_\time} \leq &\norm{\state_\time - \state^*}^2 -2\wqsmscale\step\exof{\norm{\state_\time -\step\op_\time - \state^*}^2\given\filter_\time} +2\wqsmshift\step \\
&-2\step^2\exof{\inner{\op(\state_{\time+1/2}) + \noise_{\time+1/2}(\state_{\time+1/2})}{\op_\time}\given\filter_\time}\\
&+ \step^2\exof{\norm{\rvar_\time}^2\given\filter_\time}\\
\leq &\norm{\state_\time - \state^*}^2  +2\wqsmshift\step  -2\step^2\exof{\inner{\rvar_\time}{\op_\time}\given\filter_\time} + \step^2\exof{\norm{\rvar_\time}^2\given\filter_\time}.
\end{align*}
By using the identity $\norm{a-b}^2 = \norm{a}^2+\norm{b}^2-2\inner{a}{b}$ we get
\begin{align}\label{eq:eq3}
 \exof{\norm{\hat{\state}_{\time+1}-\state^*}^2\given\filter_\time} \leq   \norm{\state_\time - \state^*}^2  +2\wqsmshift\step +\step^2\exof{\norm{\rvar_\time - \op_\time}^2\given\filter_\time} - \step^2\exof{\norm{\op_\time}^2\given\filter_\time}.
\end{align}
Furthermore, by using \cref{fact:norm} and \cref{asm:lipschitz-lineargrowth} we have that
\begin{align*}
    \norm{\op_\time- \rvar_\time}^2 &=\norm{\op(\state_\time) -\op(\state_{\time}-\step\op_\time )+\noise_\time(\state_\time) - \noise_{\time+1/2}(\state_{\time+1/2})}\\
    &\leq 3\parens*{\norm{\op(\state_\time) -\op(\state_{\time}-\step\op_\time)}^2 + \norm{\noise_\time(\state_\time)}^2 + \norm{\noise_{\time+1/2}(\state_{\time+1/2})}^2}\\
    &\leq 3\parens*{\lips^2\step^2\norm{\op_\time}^2 + \norm{\noise_\time(\state_\time)}^2 + \norm{\noise_{\time+1/2}(\state_{\time+1/2})}^2}.
\end{align*}
Thus \cref{eq:eq3} becomes
\begin{align*}
    \exof{\norm{\hat{\state}_{\time+1}-\state^*}^2\given\filter_\time} &\leq   \norm{\state_\time - \state^*}^2  +\step^2(3\lips^2\step^2-1)\exof{\norm{ \op_\time}^2\given\filter_\time} + 2(\wqsmshift\step+3\sbound^2\step^2),
\end{align*}
where we also used \cref{asm:oracle} to bound the variance of the noises $\noise_\time, \noise_{\time+1/2}$. Now if $\step\leq \dfrac{1}{\sqrt{3}\lips}$ we have that 
\begin{equation*}
    \exof{\norm{\hat{\state}_{\time+1}-\state^*}^2\given\filter_\time} \leq  \norm{\state_\time - \state^*}^2  + 2(\wqsmshift\step+3\sbound^2\step^2).
\end{equation*}
Finally, notice that
\begin{align*}
    \exof{\norm{\hat{\state}_{\time+1}-\state^*}^2\given\filter_\time} &= \norm{\state_\time - \state^*}^2 -2\step\exof{\inner{\rvar_\time}{\state_\time-\state^*} \given\filter_\time}+\step^2\exof{\norm{\rvar_\time}^2\given\filter_\time}\\
     &\leq    \norm{\state_\time - \state^*}^2  + 2(\wqsmshift\step+3\sbound^2\step^2).
\end{align*}
Thus,
\begin{equation*}
\step^2\exof{\norm{\rvar_\time}^2\given\filter_\time}\leq   2\step\exof{\inner{\rvar_\time}{\state_\time-\state^*} \given\filter_\time} +2(\wqsmshift\step+3\sbound^2\step^2).
\end{equation*}
\end{proof}

The above proposition shows how the energy descent inequality is weaken due to noise introduced by the noisy oracle. The next proposition aims to analyze the drift, \ie $\text{drift}_\time = \step\exof{\inner{\rvar_\time}{\state_\time-\state^*} \given\filter_\time}$.

\begin{proposition}\label{prop:drift}
Consider that \eqref{eq:SEG} is run and let $\text{drift}_\time = \step\exof{\inner{\rvar_\time}{\state_\time-\state^*}\given\filter_\time}$, where $\rvar_\time$ is defined as in \cref{prop:auxiliary}. If $\step < \dfrac{1}{2\wqsmscale +\sqrt{3}\lips }$ and \cref{asm:lipschitz-lineargrowth,asm:wqsm,asm:oracle} holds then 
\begin{equation}
    -\text{drift}_\time \leq -\dfrac{\wqsmscale\step}{2}\norm{\state_\time -\state^*}^2  +(\step\wqsmshift +3\step^2\sbound^2).
\end{equation}
\end{proposition}
\begin{proof}
    Recall that $\rvar_\time =\op_{\time+1/2} = \op(\state_{\time+1/2}) + \noise_{\time+1/2}(\state_{\time+1/2})$. We have
    \begin{align*}
        -\text{drift}_\time &= -\step\exof{\inner{\rvar_\time}{\state_\time-\state^*} \given\filter_\time}\\
        &= - \step\exof{\inner{\op(\state_{\time+1/2}) + \noise_{\time+1/2}(\state_{\time+1/2})}{\state_\time-\state^*} \given\filter_\time}\\
        &= -\step\exof{\inner{\op(\state_{\time}-\step\op_\time) }{\state_\time-\state^*} \given\filter_\time}\\
        &=- \step\exof{\inner{\op(\state_{\time}-\step\op_\time) }{\state_\time-\step\op_\time - \state^*} \given\filter_\time} -\step^2\exof{\inner{\op(\state_{\time}-\step\op_\time) } {\op_\time } \given\filter_\time} \\
        &\leq -\wqsmscale\step\exof{\norm{\state_\time-\step\op_\time - \state^*}^2\given\filter_\time} +\step\wqsmshift -\step^2\exof{\inner{\op(\state_{\time}-\step\op_\time) } {\op_\time } \given\filter_\time},
    \end{align*}
    where we used the fact that $\state_{\time+1/2} = \state_\time - \step\op_\time$ and the property of \acl{WQSM} (\cref{asm:wqsm}). We now use again the identity $\norm{a-b}^2 = \norm{a}^2 + \norm{b}^2 -2\inner{a}{b}$,  for all $a,b\in\R^d$, \cref{fact:norm}  and we get
    \begin{align*}
       -\text{drift}_\time \leq   &-\wqsmscale\step\exof{\norm{\state_\time-\step\op_\time - \state^*}^2\given\filter_\time} +\step\wqsmshift\\
       &-\dfrac{\step^2}{2} \parens*{\exof{\norm{\rvar_\time}^2\given\filter_\time} + \exof{\norm{\op_\time}^2 \given\filter_\time} - \exof{\norm{\rvar_\time -\op_\time}^2\given\filter_\time}}\\
       \leq &-\wqsmscale\step\exof{\norm{\state_\time-\step\op_\time - \state^*}^2\given\filter_\time} +\step\wqsmshift\\
       &-\dfrac{\step^2}{2} \parens*{\exof{\norm{\rvar_\time}^2\given\filter_\time} + \exof{\norm{\op_\time}^2 \given\filter_\time}}\\
       &+\dfrac{\step^2}{2} \exof{3\parens*{\norm{\op(\state_{\time+1/2}) -\op(\state_\time)}^2 +\norm{\noise_\time(\state_\time)}^2 + \norm{\noise_{\time+1/2}(\state_{\time+1/2}}^2}\given\filter_\time}\\
       \leq &-\wqsmscale\step\exof{\norm{\state_\time-\step\op_\time - \state^*}^2\given\filter_\time} +\step\wqsmshift\\
       &-\dfrac{\step^2}{2} \parens*{\exof{\norm{\rvar_\time}^2\given\filter_\time} + \exof{\norm{\op_\time}^2 \given\filter_\time}}\\
       &+\dfrac{3\step^2}{2}\parens*{\lips^2\step^2\exof{\norm{\op_{\time}}^2 \given\filter_\time} +2\sbound^2}.
    \end{align*}
Furthermore, it holds that $\norm{a-b}^2\geq \dfrac{\norm{a}^2}{2} - \norm{b}^2$ for all $a,b\in\R^2$; thus by using this inequality and rearranging we have
\begin{align*}
    -\text{drift}_\time \leq &-\dfrac{\wqsmscale\step}{2}\norm{\state_\time -\state^*}^2  +\step\wqsmshift +3\step^2\sbound^2\\
       & -\dfrac{\step^2}{2} \parens*{1-2\wqsmscale\step -3\step^2\lips^2} \exof{\norm{\op_\time}^2 \given\filter_\time}\\
       &-\dfrac{\step^2}{2} \exof{\norm{\rvar_\time}^2\given\filter_\time}\\
       \leq & -\dfrac{\wqsmscale\step}{2}\norm{\state_\time -\state^*}^2  +\step\wqsmshift +3\step^2\sbound^2\\
       & -\dfrac{\step^2}{2} \parens*{1-2\wqsmscale\step -3\step^2\lips^2} \exof{\norm{\op_\time}^2 \given\filter_\time}.
\end{align*}
In order to cancel out the last term of the above inequality, we require that $1-2\wqsmscale\step -3\step^2\lips^2\geq 0$ or equivalently $\step\in(-\frac{\wqsmscale+\sqrt{\wqsmscale^2 +3\lips^2}}{3\lips^2}, \frac{-\wqsmscale +\sqrt{\wqsmscale^2+3\lips^2}}{3\lips^2})$. Since $\step >0$, we need that 
\begin{align*}
    0<\step &\leq \dfrac{-\wqsmscale +\sqrt{\wqsmscale^2+3\lips^2}}{3\lips^2}\\
    &= \dfrac{3\lips^2}{3\lips^2 (\wqsmscale + \sqrt{\wqsmscale^2 + 3\lips^2})}\\
    &= \dfrac{1}{\wqsmscale + \sqrt{\wqsmscale^2 + 3\lips^2}}.
\end{align*}
Thus, if $\step\leq \dfrac{1}{2\wqsmscale + \sqrt{3}\lips}$ we get
\begin{equation*}
   -\text{drift}_\time \leq  -\dfrac{\wqsmscale\step}{2}\norm{\state_\time -\state^*}^2  +(\step\wqsmshift +3\step^2\sbound^2).
\end{equation*}
\end{proof}

With this machinery at hand we proceed to prove the following lemma.
\begin{lemma}\label{lem:ratesSEG}
    Suppose that \cref{asm:solutionset,asm:lipschitz-lineargrowth,asm:wqsm,asm:oracle} hold then the  iterations $(\state_\time)_{\time\geq \start}$ of  \eqref{eq:SEG}, if the step-size $\step \leq\dfrac{1}{2\wqsmscale +\sqrt{3}\lips} $, satisfy:
\begin{equation}
    \exof{\norm{\state_{\time +1} -\state^*}^2\given\filter_{\time}} \leq (1-\const)^\time\norm{\state_\start-\state^*}^2 +\constalt
\end{equation} 
for some constants $\const\in (0,1)$ and $\constalt\in(0,+\infty)$ that depend on the choice of step-size, as well as the parameters of the problem.
\end{lemma}
\begin{proof}
    We start by analyzing the norm of the difference between the iteration $\state_{\time
+1}$ and the solution $\state^*$. For the updates of \eqref{eq:SEG} we use $\step$ to denote the step-size and $\stepalt$ to denote the scaling parameter and drop the exponent (SEG) for simplicity.
\begin{align*}
    \norm{\state_{\time+1} -\state^*}^2 &= \norm{\state_\time -\stepalt\step\op_{\time+1/2}-\state^*}^2\\
 &= \norm{\state_\time -\state^*}^2 -2\stepalt\step\inner{\op_{\time+1/2}}{\state_\time -\state^*} + \alpha^2\step^2\norm{\op_{\time+1/2}}^2.
\end{align*}
Now by taking the expectation on both sides given the filtration $\filter_\time$ we get
\begin{align*}
\exof{\norm{\state_{\time+1} -\state^*}^2\given\filter_\time} = \norm{\state_\time -\state^*}^2  -2\stepalt\text{drift}_\time + \alpha^2\step^2\exof{\norm{\rvar_\time}^2\given\filter_\time}.
\end{align*}
where $\rvar,\text{drift}$ were defined in \cref{prop:auxiliary,prop:drift}. Now from the same propositions we get that
\begin{align}
  \exof{\norm{\state_{\time+1} -\state^*}^2\given\filter_\time} &\leq   \norm{\state_\time -\state^*}^2  -2\stepalt\text{drift}_\time + 2\alpha^2 \text{drift}_\time +2\alpha^2(\wqsmshift\step+3\sbound^2\step^2) \nonumber \\
  &\leq \norm{\state_\time -\state^*}^2 -2\stepalt(1-\stepalt)\text{drift}_\time +2\alpha^2(\wqsmshift\step+3\sbound^2\step^2) \nonumber \\
  &\leq \norm{\state_\time -\state^*}^2\parens*{1-\stepalt(1-\stepalt)\step\wqsmscale} + 2\alpha(3\step^2\sbound^2 +\step\wqsmshift). \label{eq:SEG_one_step}
  \end{align}
  Now let $\const = \stepalt(1-\stepalt)\step\wqsmscale$ and $\constalt = \frac{2\alpha(3\step^2\sbound^2 +\step\wqsmshift)}{\const}$. Since $\step \leq\dfrac{1}{2\wqsmscale +\sqrt{3}\lips} < \dfrac{1}{2\wqsmscale}$, it holds that $\const <1$.Thus, we have
  \begin{equation}
      \exof{\norm{\state_{\time+1} -\state^*}^2} \leq \parens*{1-\const}^\time\norm{\state_\start -\state^*}^2 + \constalt
  \end{equation}
  and the proof is completed.
\end{proof}

\begin{theorem}[Restated Theorem~\ref{thm:convergence}]\label{app:thm:convergence}
 Consider that either \eqref{eq:SGDA}  or \eqref{eq:SEG} is run with a stochastic oracle satisfying \cref{asm:solutionset,asm:lipschitz-lineargrowth,asm:wqsm,asm:oracle} respectively with step-sizes $\steps{SGDA} < \dfrac{\wqsmscale}{\growth^2}$, $\steps{SEG}< \dfrac{1}{2\wqsmscale+\sqrt{3}\lips}$ and $\stepsalt{SEG}\in(0,1)$ and let $(\state_\run)_{\run\geq \start}$ be the iterations generated. Then, there exists a pair of constants
 $\cnstsminor{SGDA,SEG}$ that depend on the choice of step-sizes, as well as the parameters of the model, with $\cnstminor{SGDA,SEG}\in(0,1)$ and $\cnstminoralt{SGDA,SEG}\in(0,+\infty)$ such that \begin{equation}
    \exof{\norm{\state_{\run+1}-\state^*}^2}  \leq \parens*{1-\cnstminor{SGDA,SEG} }^\run\norm{\state_{\start}-\state^*}^2 + \cnstminoralt{SGDA,SEG}, 
 \end{equation}
 for any initial point $\state_0\in\R^d$.
 \end{theorem}
 \begin{proof}
     Proof follows by combining Lemma~\ref{lem:exponentialSGDA} and \ref{lem:ratesSEG}.
 \end{proof}

 \subsection{One-step  quasi-descent inequality}
In this subsection, we provide the proof for one-step “quasi-descent” inequality stated in \cref{cor:energy}.

  \begin{corollary}[Restated Corollary \ref{cor:energy}]\label{cor:energy_app} 
 Under the conditions of \cref{thm:convergence},  for all  $\state^*\in\sols$ there exists an extended real-valued function $\energy:\R^d\to[1,\infty]$ and constants $\cnstminor{SGDA,SEG}\in(0,1),\cnstminoralt{SGDA,SEG}\in(0,\infty)$ such that  
 \begin{equation*}
    \exof{\energy(\state_{\run+1},\state^*)  \given\filter_\run}\leq  \cnstminor{SGDA,SEG}\energy(\state_\run,\state^*) + \cnstminoralt{SGDA,SEG}.
 \end{equation*}
 Specifically, $\energy(\state_\run,\state^*) = \norm{\state_{\run}-\state^*}^2 +1$.
 \end{corollary}

 \begin{proof} For \eqref{eq:SGDA}, by \cref{eq:SGDA_one_step} in the proof of \cref{lem:exponentialSGDA}, we have
  \begin{equation*} 
  \exof{\norm{\state_{\time +1} -\state^*}^2+1\given\filter_\time} \leq (1-2\wqsmscale\step+2\step^2\growth^2)\left(\norm{\state_\time -\state^*}^2+1\right)  +(2\wqsmshift\step +2\wqsmscale\step+ 2\step^2\growth^2\solb^2 + \step^2\sbound^2).
\end{equation*}
Let $c_1=1-2\wqsmscale\step+2\step^2\growth^2$ and $c_2=2\wqsmshift\step +2\wqsmscale\step+ 2\step^2\growth^2\solb^2 + \step^2\sbound^2$. By the step-size condition, we have $c_1\in (0,1)$ and $c_2\in (0,\infty)$ and thus complete the proof for \eqref{eq:SGDA}. 

 For \eqref{eq:SEG}, by \cref{eq:SEG_one_step} in the proof of \cref{lem:ratesSEG}, we have
  \begin{equation*} 
  \exof{\norm{\state_{\time +1} -\state^*}^2+1\given\filter_\time} \leq \parens*{1-\stepalt(1-\stepalt)\step\wqsmscale}\left(\norm{\state_\time -\state^*}^2+1\right) + \alpha(6\step^2\sbound^2 +2\step\wqsmshift+(1-\alpha)\step\wqsmscale).
\end{equation*}
Now let $c_1=1-\stepalt(1-\stepalt)\step\wqsmscale$ and $c_2=\alpha(6\step^2\sbound^2 +2\step\wqsmshift+(1-\alpha)\step\wqsmscale).$ Similarly, by the step-size condition, we have $c_1\in (0,1)$ and $c_2\in (0,\infty)$ and thus complete the proof for \eqref{eq:SEG}.  
 \end{proof}
\clearpage

\section{Omitted Proofs of Section~\ref{sec:results}}
\subsection{Minorization  Condition and Geometric Drift Property}

\addcontentsline{toc}{subsubsection}{Proof of Lemma~\ref{lem:minor} (Minorization Condition)}
\begin{lemma}[Restated Lemma~\ref{lem:minor}]\label{lem:minor-app}
   Let \cref{asm:solutionset,asm:lipschitz-lineargrowth,asm:wqsm,asm:oracle,asm:uniformbound} be satisfied for \eqref{eq:SGDA} and \eqref{eq:SEG}. Then given the step-sizes specified in \cref{thm:convergence} it holds that for both algorithms  the minorization condition is satisfied. 
   Namely there exist constant $\delta>0$, probability measure $\nu$ and set $C$, dependent on the algorithm such that $\nu(C)=1$ and $\nu(C^c)=0$ such that 
\begin{equation} \label{eq:main:minor-app}
    \Pr[\next\in A|\state_\run=x]\geq \delta\one_{\set}(x)\nu(A) 
    \quad \text{ for all }\;\;A\in\borel(\R^d), x\in \R^d.
\end{equation}
\end{lemma}

\begin{proof}
    We again split the proof in two different parts for each one of the two algorithms. For the sequence we fix a point $\state^*\in\sols$ and we consider the energy function defined as $\energy(\state) = \norm{\state-\state^*}^2 +1$.
    
    \paragraph{SGDA:} We start by observing that the Energy/Lyapunov function $\energy(\point) := \norm{\point-\state^*}^2 +1 $ is a function unbounded off small sets, \ie the sublevel sets $C(r):=\braces{\point\in\R^d | \energy(\point)\leq r}$ are either empty or small for all $r>0$. Indeed assume that $C(r) = \braces{\point\in\R^d | \energy(\point)\leq r}$ is non-empty $(r>1)$, then the sublevel sets correspond to some ball $\ball(\state^*, \sqrt{r-1})$ for $r>1$. We will prove that the ball $\ball(\state^*, \sqrt{r-1})$ for $r>1$ is actually $\nu_1$-small for $m=1$ (see \cref{def:small}).
    \begin{align*}
        P(x,B) &= \prob\parens*{\state_{\time+1}\in B|\state_\time=\point}\\
        &=\prob\parens*{\state_\time -\step(\op(\state_\time) +\noise_\time(\state_\time))\in B|\state_\time = \point}     \\
        &=\prob\parens*{\point - \step(\op(\point) +\noise_\time(\point))\in B}\\
        &=\prob\parens*{\noise_\time(\point)\in (\dfrac{\point}{\step}-\op(\point))+(-\dfrac{1}{\step}B)},
    \end{align*}
where $\noise_\time(\point) \sim\law(\noise(\point))$ for all $\run\geq \start$. With this notation we want to emphasize that once $\state_\time$ is fixed the distribution of the noise is independent of the time-step, since we have assumed that at each time-step the noises are \acl{iid} random fields. Thus, we have
\begin{align}\label{eq:irr-SGDA}
    P(x,B) &= \int_{\beta\in B} \pdf_{\noise(\point)}\parens{\dfrac{\point-\beta}{\step}-\op(\point)}\dd\beta\\
    &\geq \int_{\beta\in B} \inf_{\point\in C(r)}\pdf_{\noise(\point)}\parens{\dfrac{\point-\beta}{\step}-\op(\point)}\dd\beta\\
    &:= \meas{SGDA}_r(B).
\end{align}
Notice that $\meas{SGDA}_r$ is a non-trivial measure since if we set $B=C(r)$, which is a non-empty and bounded set, we have
\begin{equation*}
    \meas{SGDA}_r(C(r)) = \int_{\pointalt\in C(r)} \inf_{\point\in C(r)}\pdf_{\noise(\point)}\parens*{\dfrac{\point-\pointalt}{\step}-\op(\point)}\dd\pointalt >0 ,
    \end{equation*}
    which follows from \cref{asm:uniformbound}. 

    We now fix $r=r_0 >1 $ and proceed in proving the minorization property.
    Consider the measure $\tilde{\nu}^{\text{\tiny SGDA}}_{r_0}(X) = \one\parens{X\subseteq C(r_0)}\dfrac{\meas{SGDA}_{r_0}(X)}{\meas{SGDA}_{r_0}(C(r_0))}$ for all $X\in\borel(\R^d)$
    . It is easy to verify that $\tilde{\nu}^{\text {\tiny SGDA}}_{r_0}(C(r_0)) =1$ and $\tilde{\nu}^{\text {\tiny SGDA}}_{r_0}(C(r_0)^c) = 0$. Additionally, if $\braces{\point\notin C(r_0)\text { or } A\nsubseteq C(r_0)}$ we have that $P(x,A) \geq \delta \one_{C(r_0)} (x) \tilde{\nu}^{\text \tiny SGDA}_{r_0}(A)=0$. Also, if $\braces{\point\in C(r_0)\text { and } A\subseteq C(r_0)}$  we have $P(x,A) \geq \meas{SGDA}_{r_0}(A) = \delta \one_{C(r_0)}(x) \tilde{\nu}^{\text \tiny SGDA}_{r_0}(A)$, where $\delta = \meas{SGDA}_{r_0}(C(r_0)) > 0$ and thus the proof is completed.

    \paragraph{SEG:} We continue with the proof when \eqref{eq:SEG} is run. Similarly as before we have that 
    \begin{align*}
        P(x,B) =& \prob\parens{\state_{\run+1}\in B|\state_\run = x}\\
        =&\prob\parens{\state_{\run} - \stepalt\step\op(\state_\run -\step\op(\state_\run)
        -\step\noise_\run(\state_\run))\\
        &-\stepalt\step\noise_{\run+1/2}(\state_\time -\step\op(\state_\time) -\step\noise_\time(\state_\time))\in B|\state_\time =x}\\
        =& \prob\parens{\point - \stepalt\step\op(\point -\step\op(\point) -\step\noise_\time(\point))\\
        &-\stepalt\step\noise_{\time+1/2}(\point-\step \op(\point)-\step\noise_\time(\point))\in B}
    \end{align*}
    where $\noise_\time(\point)\sim \law(\noise^A(\point))$, $\noise_{\time+1/2}(\point)\sim\law(\noise^{B}(\point))$ and $\noise^A(\point)\perp\noise^B(\point)$, identically distributed. Thus,
    \begin{align*}
        P(x,B) &= \int_{\snoise\in\R^d} \pdf_{\noise^A(\point)}\parens*{\snoise} \prob\parens*{\point -\stepalt\step\op(\point-\step\op(\point)-\step\snoise) -\stepalt\step\noise^B(\point-\step\op(\point) -\step\snoise)\in B}\dd\snoise\\
        &=\int_{\beta\in B}\int_{\snoise\in\R^d} \pdf_{\noise^A(\point)}\parens*{\snoise} \pdf_{\noise^B(\point -\step\op(\point)-\step\snoise)} \parens*{\dfrac{\point -\beta}{\stepalt\step} - \op(\point-\step\op(\point)-\step\snoise)}\dd\snoise\dd\beta\\
        &\geq \int_{\beta\in B}\int_{\snoise\in\ball(0,1)}\pdf_{\noise^A(\point)}\parens*{\snoise} \pdf_{\noise^B(\point -\step\op(\point)-\step\snoise)} \parens*{\dfrac{\point -\beta}{\stepalt\step} - \op(\point-\step\op(\point)-\step\snoise)}\dd\snoise\dd\beta.
    \end{align*}
    Notice that since $\point\in C(r)$, we have that $\point-\step\op(\point)-\step\snoise\in C(r) - \step \op(C(r)) - \step \ball(0,1)$. Thus $\pdf_{\noise^A(\point)}(t) \geq \inf_{\point\in C(r)}\pdf_{\noise^A(\point)} (t)>0$ for all $t\in\R^d$ and $\pdf_{\noise^B(\point-\step\op(\point)-\step\snoise)}(t) \geq \inf_{\rho\in C(r) - \step \op(C(r)) - \step \ball(0,1)}\pdf_{\noise^B(\rho)}(t)$. Hence, we can define the following measure for any set $B$:
    \begin{equation*}
        \meas{SEG}_{r_0}(B) := \int_{\beta\in B}\int_{\snoise\in\ball(0,1)} \inf_{\point\in C(r)}\pdf_{\noise^A(\point)}\parens{\snoise}\inf_{\rho\in C'} \pdf_{\noise^B(\rho)}\parens*{\dfrac{x-\beta}{\stepalt\step}-\op(\rho)}\dd\snoise\dd\beta, 
    \end{equation*}
    where $C' = C(r) - \step \op(C(r)) - \step \ball(0,1)$. Notice that the measure is non-trivial since for some fixed $r=r_0>1$ we have that $\meas{SEG}_{r_0} (C(r_0)) >0$ since $C(r_0)$ is non-empty. As in the case of SGDA we define 
    \begin{equation*}
        \tilde{\nu}^{\text{\tiny SEG} }_{r_0}(X) = \one(X\subseteq C(r_0)) \dfrac{\meas{SEG}_{r_0}(X)}{\meas{SEG}_{r_0}(C(r_0))}.
    \end{equation*}
    Thus, we have that 
    \begin{equation*}
        P(x,B) \geq \tilde{\nu}^{\text{\tiny SEG} }_{r_0}(B).
    \end{equation*}
    By repeating the exact same methodology as before the result follows.
    \end{proof}

\addcontentsline{toc}{subsubsection}{Proof of Corollary~\ref{cor:geomdrift} (Geometric Drift Property)}
\begin{corollary}[Improved version of Corollary~\ref{cor:geomdrift}]\label{cor:geomdrift-app}
    Under the setting of \cref{lem:minor} the functions $f_1:=\energy,f_2:=\sqrt{\energy}$, $f_1,f_2:\R^d\to \R_{\geq 0}$ presented in \cref{cor:energy} satisfies the (V4) Geometric Drift Property of \cite{Meyn12_book}  for the Markov Chain generated either by \eqref{eq:SGDA} or \eqref{eq:SEG}. Namely it holds that there exist a measurable set $C$, and constants $\beta >0$, $b<\infty$ such that 
    \begin{equation}\label{eq:main:drift-app}
        \Delta f_i(x) \leq -\beta f_i(x) + b\one_C(x), x\in \R^d,
    \end{equation}
    where $\Delta f_i(x) = \int_{y\in \R^d} P(z,dy)f_i(y) - f_i(x)$ for $i\in\{1,2\}$.
\end{corollary}
\begin{proof}
  Based on \cref{def:drift} we need to show that there exists a function $f:\R^d\to[1,\infty)$, a measurable set $C$ and constants $\beta>0, b<\infty$ such that $\Delta f(x) \leq -\beta f(x) + b\one_C(x)$ for all $x\in\R^d$. We start with the observation that 
\[
       \Delta f(x) = \int_{y\in\R^d}P(x,\dd y)f(y) - f(x)=\exof{f(\state_{\time+1}) -f(\state_\time)\given\filter_\time:\braces{\state_\time = x}}
\]
where $\state_\time$ that is generated either through \eqref{eq:SGDA} or \eqref{eq:SEG}. Furthermore, notice that the function defined in \cref{cor:energy}, $\energy:\R^d\to[1,\infty)$ is extended-real valued and also it holds that   
    \begin{equation*}
        \exof{\energy(\state_{\time+1})\given \filter_\time: \braces{\state_\time = x}} \leq \cnstminor{SGDA,SEG}\energy(x) + \cnstminoralt{SGDA,SEG}
    \end{equation*}
    with $\cnstminor{SGDA,SEG}\in (0,1)$ and $\cnstminoralt{SGDA,SEG}\in(0,+\infty)$. 
    
    Similarly, for the function $\sqrt{\energy}$ we have that
    \begin{align*}
        \exof{\sqrt{\energy(\state_{\time+1})}\given \filter_\time: \braces{\state_\time = x}} &\leq \sqrt{\exof{\energy(\state_{\time+1})\given \filter_\time: \braces{\state_\time = x}} }\\
        &\leq\sqrt{\cnstminor{SGDA,SEG}\energy(x) + \cnstminoralt{SGDA,SEG}}\\
        &\leq \sqrt{\cnstminor{SGDA,SEG}}\sqrt{\energy(x) } + \sqrt{\cnstminoralt{SGDA,SEG}}.
    \end{align*}
    Now notice that for any function $\energy$ which is unbounded off small sets and for all $x\in\R^d$ satisfies 
    \begin{align*}    \exof{\energy(\state_{\time+1})\given\filter_\time:\braces{\state_\time = x} }&\leq \const \energy(x) +\constalt,
    \end{align*}
    or equivalently
    \begin{equation*}     \exof{\energy(\state_{\time+1})\given\filter_\time:\braces{\state_\time = x}} -\energy(x) \leq -(1-\const)\energy(x) +\constalt,
    \end{equation*}
    we have that it satisfies the geometric drift property for any set $C = \braces{x\in\R^d : \energy(x) \leq \dfrac{2\constalt}{(1-\const)}}$ and constants $\beta = \dfrac{1-\const}{2}$ and $b = \constalt$.  Indeed, 
    \begin{equation*}
        \constalt \leq \one_C(x)\constalt + \one_{C^c}(x)\dfrac{1-\const}{2}\energy(x) \text{ for all }x\in\R^d.
    \end{equation*}
    Thus, 
    \begin{align*}
        \exof{\energy(\state_{\time+1})\given\filter_\time:\braces{\state_\time = x}} -\energy(x) &\leq -(1-\const)\energy(x) + \one_C(x)\constalt   + \one_{C^c}(x)\dfrac{1-\const}{2}\energy(x)\\
        &\leq -\dfrac{1-\const}{2}\energy(x) + \one_C(x)\constalt.
    \end{align*}
    The last inequality follows from the fact that $\one_{C^c}(x) \leq 1$ and $\const\in(0,1)$.
\end{proof}

\subsection{Invariant Measure, Total Variation Convergence and Limit Theorems }

\addcontentsline{toc}{subsubsection}{Proof of Lemma~\ref{lem:properties} (Irreducibility,Recurrence,Aperiodicity)}
\begin{lemma}[Restated Lemma~\ref{lem:properties}]\label{lem:properties-app}
    The corresponding Markov chain sequences $(\state_\run)_{\run\geq \start}$ for \eqref{eq:SGDA} and \eqref{eq:SEG} have the following properties:
    \begin{itemize}[noitemsep,nolistsep,leftmargin=*]
        \item They are $\irr-$irreducible for some non-zero $\sigma$-finite measure $\irr$ on $\R^d$ over Borel $\sigma$-algebra of $\R^d$.  
        \item They are strongly aperiodic. 
        \item They are Harris and positive recurrent with an invariant measure.
    \end{itemize}
\end{lemma}
\begin{proof}
    We prove each one of the properties above separately. 
    \begin{itemize}
        \item \textbf{(Irreducible):} Consider any non-zero  $\sigma$-finite measure $\varphi$ in Borel $\sigma$-algebra of $\R^d$. From the proof of \cref{lem:minor-app} for \eqref{eq:SGDA} we have
        \begin{equation*}
            \prob\parens{\state_{\time+1}\in A|\state_\time = x} =\int_{a\in A}\pdf_{\noise(x)}(\dfrac{x-a}{\step}-\op(x))\dd a.
        \end{equation*}
        By \cref{asm:uniformbound} and for any $A\subseteq\borel(\R^d)$ with $\irr(A) >0$ we  have that $\braces{x}\subseteq \ball(x,1)$ and there exists $\varepsilon>0$ such that $\ball(a_0,\varepsilon)\subseteq A$, for some $a_0\in A$. Thus,
        \begin{align*}
            P(x,A) &\geq \int_{\tilde{a}\in\ball(a_0,\varepsilon)} \pdf_{\noise(x)}(\dfrac{x-\tilde{a}}{\step}-\op(x))\dd \tilde{a}\\
            &\geq \int_{\tilde{a}\in\ball(a_0,\varepsilon)} \inf_{\tilde{x}\in\ball(x,1)}\pdf_{\noise(\tilde{x})}(\dfrac{\tilde{x}-\tilde{a}}{\step}-\op(x))\dd \tilde{a}\\
            &>0.
        \end{align*}
        Similarly, for the case of \eqref{eq:SEG} and by repeating the same argument for some non-zero  $\sigma$-finite measure $\varphi$ in $\borel{\R^d}$ algebra, we have that 
        \begin{align*}
            P(x,A) &= \int_{a\in A} \int_{\snoise\in\ball(0,1)}\pdf_{\noise^A(x)}(\snoise)\pdf_{\noise^B(x-\step\op(x)-\step\snoise)}\parens*{\dfrac{x-a}{\stepalt\step}-\op(x-\step\op(x)-\step\snoise)}\dd \snoise\dd a\\
            &\geq \int_{\tilde{a}\in\ball(a_0,\varepsilon)}\int_{\snoise\in\ball(0,1)}\inf_{\tilde{x}\in\ball(x,1)}\pdf_{\noise^A(\tilde{x})}(\snoise)\inf_{\rho\in C}\pdf_{\noise^B(\rho)}\parens*{\dfrac{\tilde{x}-\tilde{a}}{\stepalt\step}-\op(\rho)}\dd\snoise\dd \tilde{a}\\
            &>0,
        \end{align*}
        where $C = \ball(x,1) -\step\op(\ball(x,1)) -\step\ball(0,1)$. The strict positivity for both cases follows from \cref{asm:uniformbound}. Thus, by \cref{def:irr} the sequences are $\irr$-irreducible.
        
        \item \textbf{(Strongly Aperiodic):} This is an immediate consequence of the proof of \cref{lem:minor-app}, since the sets $C(r)$ are  small and have positive measure for the measure we constructed.
        \item \textbf{(Recurrent with invariant measure):} Given that the Markov chain is $\psi$-irreducible and aperiodic, from Theorem 15.0.1 (Geometric Ergodic Theorem) in \cite{Meyn12_book}  we have that the chain is positive recurrent and has an invariant measure. This is true since we have proven the geometric drift property (cf.~\cref{cor:geomdrift-app}) for a small set, which is also a petite set by \cref{prop:petite}. 

        The fact that the Markov chain is also Harris is a consequence of Theorem 9.1.8 of \cite{Meyn12_book}. For completeness, we mention here that if a chain is $\irr$-irreducible and there exists a function $f$ that is unbounded off petite sets such that $\Delta f \leq 0$ then the chain is Harris recurrent. All these requirements are direct implications of the results presented so far, particularly the proof of Corollary \ref{cor:geomdrift-app} and the current lemma. As such, the Markov chains induced by the stochastic gradient descent models in Equations \eqref{eq:SGDA} and \eqref{eq:SEG} are demonstrably Harris recurrent.  \end{itemize}
        
        \end{proof}

\addcontentsline{toc}{subsubsection}{Proof of Theorem~\ref{thm:close} (Unique Invariance, Geometric convergence under TV)}

\begin{theorem}[Restated Theorem \ref{thm:close}]\label{thm:close-app}
    Let \cref{asm:solutionset,asm:lipschitz-lineargrowth,asm:wqsm,asm:oracle,asm:uniformbound} be satisfied for \eqref{eq:SGDA} and \eqref{eq:SEG}. Then given the step-sizes specified in \cref{thm:convergence} it holds that
    \begin{enumerate}
        \item \eqref{eq:SGDA} and \eqref{eq:SEG} iterates admit a unique stationary distribution $\distr{SGDA,SEG}{\step}\in\mathcal{P}_2(\R^d).$
        \item For a test function $\tf:\R^d\to\R$ satisfying that $\abs{\tf(\point)}\leq \growth_{\tf} (1+ \norm{\point})$ for all $\point\in\R_{\geq 0}^{d}$, for some $\growth_\tf >0$ and for any initialization $\state_\start\in\R^d$ there exist $\tconst{SGDA,SEG}_{\tf,\step}\in(0,1)$ and $\tconstalt{SGDA,SEG}_{\tf,\state_\start,\step}\in(0,\infty)$ such that:
        \begin{equation}
            \label{eq:convergence_rate-app}
           \abs*{\mathbb{E}_{\state_{\run}}\bracks{\tf(\state_\run)} - \mathbb{E}_{\point\sim\distr{SGDA,SEG}{\step}}\bracks{\tf(\point)} }\leq \tconstalt{SGDA,SEG}_{\tf,\state_\start,\step}(\tconst{SGDA,SEG}_{\tf,\step})^\run.
        \end{equation}
        Hence, \eqref{eq:SGDA} and \eqref{eq:SEG} converge geometrically under the total variation distance to $\distr{SGDA,SEG}{\step}$.
    \item Finally, for any test function $\tf$ that is $\lips_\tf$-Lipschitz we have that
    \begin{equation}
       \abs{ \mathbb{E}_{\point\sim\distr{SGDA,SEG}{\step}}\bracks{\tf(\point)}-\tf(\state^*)} \leq \lips_\tf \sqrt{D^{\{\text{\tiny SGDA,SEG}}\}},
    \end{equation}
    for some constant $D^{\{\text{\tiny SGDA,SEG}\}}\propto \max(\lambda,\step)/\mu$.
    \end{enumerate}
\end{theorem}

\begin{proof}
The first part of the theorem follows from the fact that the induced Markov chains are Harris recurrent and aperiodic with invariant measure and have the geometric drift property; thus from Strong Aperiodic Ergodic Theorem (See Theorem 13.0.1 in \cite{Meyn12_book}) the measure is unique and finite.  Additionally assume that $\state_\start\sim \distr{SGDA,SEG}{\step}$. Then by the invariance property $(\state_\run)_{\run\geq \start}\sim \distr{SGDA,SEG}{\step}$. Using Corollary~\ref{cor:energy} 
for some arbitrary fixed $\state^*\in\points^*$, there exist two corresponding constants $(\cnstminor{SGDA,SEG},\cnstminoralt{SGDA,SEG})$ such that $\cnstminor{SGDA,SEG}\in(0,1)$ and $\cnstminoralt{SGDA,SEG}\in(0,\infty)$ that satisfy  
 \begin{align} \label{eq:mse_stationary_ub}
         \exof{\norm{\state_{\run+1}-\state^*}^2 +1  \given\filter_\run}&\leq  \cnstminor{SGDA,SEG}(\norm{\state_{\run}-\state^*}^2 +1) + \cnstminoralt{SGDA,SEG}\Rightarrow\\
         \ex_{\state\sim \distr{SGDA,SEG}{\step}}[\norm{\state-\state^*}^2]&\leq  \frac{\cnstminor{SGDA,SEG} + \cnstminoralt{SGDA,SEG} -1}{1-\cnstminor{SGDA,SEG}}=\bigoh(\max(\lambda,\step)/\mu)<\infty.
 \end{align}
Since $\Vert\state^*\Vert \leq \solb$, the above inequality implies that $\distr{SGDA,SEG}{\step}\in\mathcal{P}_2(\R^d).$

For the second part, we will use the geometric convergence theorem for Harris positive strongly aperiodic Markov Chains endowed with geometric drift property (See 16.0.1 in \cite{Meyn12_book})
\begin{align*}
    |\tf(\point)|&\leq \growth_{\tf} (1+ \norm{\point}) \leq
    \growth_{\tf} ((R+1)+ \norm{\point-\point^*})\leq
    \growth_{\tf} (R+1)(1+ \norm{\point-\point^*})\\&\leq
    \sqrt{2}\growth_{\tf} (R+1)\sqrt{\energy(\point)}\leq
    \max(1,\sqrt{2}\growth_{\tf} (R+1))\cdot \energy'(\point)=c'\energy'(\point)
\end{align*}
where $c':=\max(1,\sqrt{2}\growth_{\tf} (R+1))$ and 
$\energy'(\point):=\sqrt{\energy(\point)}$.
Notice that Corollary~\ref{cor:geomdrift-app} certifies that $\energy'$ also satisfies geometric drift property. Additionally, since $c'\geq 1$, $\energy''(\point):=c'\energy'(\point)$ also satisfies the geometric drift property.
Hence we can prove that \eqref{eq:SEG},\eqref{eq:SGDA} are $\energy''$-uniformly ergodic (Theorem 16.0.1 Condition (iv) in \cite{Meyn12_book}). Therefore, from the equivalent condition (ii) of the aforementioned theorem, there exist $r_{_{\ell_{\tf},\step}}\in(0,1)$, $R_{_{\ell_{\tf},\step}}\in(0,\infty)$ such that
\[ \abs{ P^{k}\tf(\point_\start) - \mathbb{E}_{\point\sim\distr{SGDA,SEG}{\step}}\bracks{\tf(\point)} }\leq R_{_{\ell_{\tf},\step}}r_{_{\ell_{\tf},\step}}^k|\energy''(\point_\start)|, \]
thus by setting $\tconstalt{SGDA,SEG}_{\tf,\state_\start,\step}:= R_{_{\ell_{\tf},\step}} |\energy''(\point_\start)| $ and $\tconst{SGDA,SEG}_{\tf,\step}:=r_{_{\ell_{\tf},\step}}$ we get the requirement. Finally for the total variation distance it suffices to address only test functions that are bounded by $1$. Thus there exist constants $r_{_{\step}}\in(0,1)$, $R_{_{\step}}\in(0,\infty)$ independent of the function such that
\[ \sup_{|\tf|\leq 1}\abs{ P^{k}\tf(\point_\start) - \mathbb{E}_{\point\sim\distr{SGDA,SEG}{\step}}\bracks{\tf(\point)} }\leq R_{_{\step}}r_{_{\step}}^k|\energy''(\point_\start)|, \]
which implies the geometric convergence under total variation distance via the dual representation of Radon metric for bounded initial conditions \cite{Wikipedia2023}.

For the last part, we start by linearity of expectation and Lipschitzness of $\tf$:
    \begin{align*}       
       \abs{ \mathbb{E}_{\point\sim\distr{SGDA,SEG}{\step}}\bracks{\tf(\point)}-\tf(\state^*)}&=
       \abs{ \mathbb{E}_{\point\sim\distr{SGDA,SEG}{\step}}\bracks{\tf(\point)-\tf(\state^*)}}\\
       &\leq \mathbb{E}_{\point\sim\distr{SGDA,SEG}{\step}}\bracks{\abs{ \tf(\point)-\tf(\state^*)}}\\
       &\leq \mathbb{E}_{\point\sim\distr{SGDA,SEG}{\step}}\bracks{\lips_\tf \norm{\point-\state^*}}\\
        &\leq \lips_\tf \sqrt{ \mathbb{E}_{\point\sim\distr{SGDA,SEG}{\step}}\bracks{ \norm{\point-\state^*}^2}}\\
       &\leq \lips_\tf \sqrt{D^{\{\text{\tiny SGDA,SEG}}\}}
    \end{align*}
    where $D^{\{\text{\tiny SGDA,SEG}\}}\propto \max(\lambda,\step)/\mu$ by \cref{eq:mse_stationary_ub}.

\end{proof}

\addcontentsline{toc}{subsubsection}{Proof of Theorems~\ref{cor:LLN} and ~\ref{thm:CLT} (LLN and CLT for Markov Chains of \eqref{eq:SEG},\eqref{eq:SGDA})}

Below we use the following notations. The distribution $\pi$ refers to  $\distr{SGDA,SEG}{\step}$ for respective algorithms. For any function $\phi':\R^d\to\R$, we introduce the shorthand
\begin{equation*}
    S_\nRuns(\phi') := \sum_{\time=1}^\nRuns \phi'(\state_\time);
\end{equation*}
in addition, we use  $\pi(\tf')$ to denote the expected value of $\tf'$ over $\pi$,  i.e.,  $\pi(\tf') = \ex_{\point\sim \pi}[\tf'(x)]$.
\begin{theorem}[Restated Theorems~\ref{cor:LLN} and ~\ref{thm:CLT}]\label{thm:app-Limits}
Let \cref{asm:solutionset,asm:lipschitz-lineargrowth,asm:wqsm,asm:oracle,asm:uniformbound} hold. 
 Then for  choice of step-sizes specified in \cref{thm:close} and any function $\tf:\R^d\to\R$ satisfying $\pi(\abs{\tf}) <\infty$,  we have that 
     \begin{equation}
       \lim_{\nRuns\to\infty}\dfrac{1}{\nRuns}S_{\nRuns}(\tf)
       =\lim_{\nRuns\to\infty}\dfrac{1}{\nRuns}\sum_{\run=\start}^\nRuns \tf(\state_\run)  
       =\pi(\tf) \quad\text{a.s.},   
       \tag{Law of Large Numbers for \eqref{eq:SGDA},\eqref{eq:SEG}}
     \end{equation}
and that
    \begin{equation}
    \nRuns^{-1/2}S_\nRuns(\tf - \pi(\tf)) \xrightarrow{d}
    \normal(0,\sdev^2_\pi(\tf)),
    \tag{Central Limit Theorem for \eqref{eq:SGDA},\eqref{eq:SEG}}
    \end{equation}
    where $\sdev^2_\pi(\tf):= \lim_{\nRuns\to\infty}\frac{1}{\nRuns}\ex_{\pi}\bracks{S_{\nRuns}^2(\tf-\pi(\tf))}.$
\end{theorem}

\begin{proof}
According to Theorem 17.0.1 in \cite{Meyn12_book}, the Law of Large Numbers 
and the Central Limit Theorem, as described in \cref{thm:app-Limits}, hold for positive Harris chains with invariant measures, given that they exhibit 
$\energy^*$-uniform ergodicity. To complete the proof, it is necessary to 
demonstrate that a function $\tf$ with linear growth fulfills the conditions 
of Theorem 17.0.1. This can be achieved by proving the existence of an energy 
function $\energy^*(\cdot)$ satisfying $(i)$ the (V4) geometric drift property in \cite{Meyn12_book} and $(ii)$ $\abs{\tf(\point)}^2\leq \energy^*(\point)$.
\begin{align*}
\abs{\tf(\point)}^2&\leq \growth_{\tf}^2 (1+ \norm{\point})^2\leq \growth_{\tf}^2 (1+ R +\norm{\point-\point^*})^2 \leq \growth_{\tf}^2 (1+ R)^2 (1+\norm{\point-\point^*})^2 \\
&\leq \sqrt{2} \growth_{\tf}^2 (1+ R)^2 \sqrt{(1+\norm{\point-\point^*}^2)}\\
&\leq 
\max(1,\sqrt{2} \growth_{\tf}^2 (1+ R)^2) \sqrt{(1+\norm{\point-\point^*}^2)}
:=\energy^*(\point)    
\end{align*}
By Corollary $\ref{cor:geomdrift-app}$, we get that $\energy^*$ satisfies geometric drift property, thus proving that \eqref{eq:SEG} and \eqref{eq:SGDA} are $\energy^*$-uniformly ergodic. We complete the proof of \cref{thm:app-Limits}.
\end{proof}

\clearpage

\section{Omitted Proofs of Section~\ref{sec:application}}
\subsection{Min-Max Convex-Concave Games}

\addcontentsline{toc}{subsubsection}{Proof of Theorem~\ref{thm:minmax} (Bias in Duality Gap)}

\begin{theorem}[Restated Theorem~\ref{thm:minmax}]\label{thm:app-minmax}
    Let \cref{asm:solutionset,asm:lipschitz-lineargrowth,asm:wqsm,asm:oracle,asm:uniformbound} hold then the iterates of \eqref{eq:SGDA}, \eqref{eq:SEG} when run with the step-sizes given in \cref{thm:convergence} admit a stationary distribution $\distr{SGDA,SEG}{\step}$ such that
    \begin{equation}
        \ex_{\point\sim\distr{SGDA,SEG}{\step}}\bracks{\gap_\op(\point)}\leq \const\steps{SGDA,SEG},
    \end{equation}
    where $\gap_\op(\point)$ is
    the restricted merit function $\gap_\op(\point):=\sup_{\point^*\in\sols}\inner{\op(\point)}{\point - \point^*}$ and    
    $\const\in\R$ is a constant  and depends on the parameters of the problem. 
\end{theorem}
\begin{proof}
From  the analysis of \eqref{eq:SGDA} in \cref{lem:exponentialSGDA} (cf.~\cref{eq:intSGDA,eq:SGDA3}) we have that     \begin{align*}
        \norm{\state_{\time+1} -\state^*}^2&\leq \snorm^2 -2\step\inner{\op(\state_\time)}{\state_\time-x^*} - 2\step\inner{\noise_\time(\state_\time)}{\state_\time-x^*}+
        \step^2\norm{\op(\state_\time) + \noise_\time(\state_\time)}^2,\\
    \norm{\op(\point)}^2&\leq 2\growth^2((1+\solb)^2 +\norm{\point-\sol}^2 ).
    \end{align*}
    Since $\ex_{\state_{\time+1}\sim\pi_\step}\bracks{\norm{\state_{\time+1} -\state^*}^2} = \ex_{\state_{\time}\sim\pi_\step}\bracks{\snorm^2}$ we have that 
    \begin{align*}
        \frac{1}{\step}\ex_{\state_{\time}\sim\pi_\step} \bracks{\inner{\op(\state_\time)}{\state_\time-\state^*}}
        &\leq 2\ex_{\state_{\time}\sim\pi_\step}\bracks{\growth^2((1+\solb)^2 +\norm{\state_{\time}-\sol}^2)} + 2\ex_{\state_{\time}\sim\pi_\step}\bracks{ \norm{ \noise_\time(\state_\time)}^2 }  )\\
        &\leq 2\growth^2((1+\solb)^2 +2\ex_{\state_{\time}\sim\pi_\step}\bracks{\norm{\state_{\time}-\sol}^2}) + 2\sbound^2\\
        &\leq 2\growth^2((1+\solb)^2 +2\const_2^{\text{\tiny SGDA}}) + 2\sbound^2\\
        &\leq\max_{\step\in(0,\dfrac{\wqsmscale}{\lips^2})} 2\growth^2((1+\solb)^2 +2\const_2^{\text{\tiny SGDA}}) + 2\sbound^2\\   
        &\leq C
    \end{align*}
    where $C = \max_{\step\in(0,\dfrac{\wqsmscale}{\lips^2})}\bracks{2\growth^2((1+\solb)^2 +2\const_2^{\text{\tiny SGDA}}) + 2\sbound^2}$ (Recall that $\const_2^{\text{\tiny SGDA}}$ depends on the step-size). 
    
    \noindent For the case of \eqref{eq:SEG}, it easy to see that $\gap_\op(\point)\leq \lips \norm{\state_\time-\state^*}^2$. So the rest of the proof is derived by \cref{thm:convergence}, using dominant convergence theorem for $\ex_{\state_{\time+1}\sim\pi_\step}\bracks{\norm{\state_{\time+1} -\state^*}^2}$, as well as the invariance property that $\state_\infty\sim \pi_\step$ if we initialize $\state_0\sim \pi_\step$.
\end{proof}

\addcontentsline{toc}{subsubsection}{Connection of $\dualitygap_f$ and $\gap_\op$; Proof of \eqref{eq:game_value_bound}}
    
We next show the connection of $\dualitygap_f$ and $\gap_\op$ for a convex-concave function $f$ and $\op=(\nabla_\theta f,-\nabla_\phi f)$:
    \begin{align*}
        \dualitygap_f(\theta,\phi) &= \max_{\phi'\in\R^{d_2}}f(\theta,\phi') -\min_{\theta'\in\R^{d_1}}f(\theta',\phi)\\
        &=(f(\theta,\phi) -\min_{\theta'\in\R^{d_1}}f(\theta',\phi) )-(f(\theta,\phi)-\max_{\phi'\in\R^{d_2}}\op(\theta,\phi'))\\
        &\leq \inner{\op(\theta,\phi)}{(\theta,\phi) - (\theta^*,\phi^*)}, 
    \end{align*}
where the last step holds since $f$ is convex (resp.\ concave) in its first (resp.\ second) argument.
Thus if we call $x = (\theta,\phi)$, $x^* = (\theta^*,\phi^*)$, we have      
    \begin{equation*}
      \dualitygap_f(\theta,\phi) \leq \gap_\op(\point).
    \end{equation*}
    
Additionally, it is easy to see that 
    \begin{align*}
        \op(\theta,\phi) \leq \max_{\phi'\in\R^{d_2}} \op(\theta,\phi')= \dualitygap(\theta,\phi) + \min_{\theta'\in\R^{d_1}}\op(\theta',\phi)\leq \dualitygap(\theta,\phi) + \max_{\phi'\in\R^{d_2}}\min_{\theta'\in\R^{d_1}}\op(\theta',\phi')
    \end{align*}
and
    \begin{align*}
        \op(\theta,\phi) \geq \min_{\theta'\in\R^{d_1}} \op(\theta',\phi)=   \max_{\phi'\in\R^{d_2}}\op(\theta,\phi')- \dualitygap(\theta,\phi)\geq -\dualitygap(\theta,\phi) + \min_{\theta'\in\R^{d_1}}\max_{\phi'\in\R^{d_2}}\op(\theta',\phi').
    \end{align*}
By applying the expectation with respect to the invariant distribution and Von-Neuman's minimax theorem we get the desired result in \cref{eq:game_value_bound}.

\clearpage

\subsection{Bias Refinement in Quasi-Monotone Operators}

\addcontentsline{toc}{subsubsection}{Proof of Lemma~\ref{lem:moments} (Fourth Moment Bounds)}
\begin{lemma}\label{lem:moments}
    In the setting of \cref{thm:bias_expansion} the moments $\text{Mom(k)} = \exof{\norm{\state_\time -\state^*}^k }$ are bounded by a function of $f_k(\step)$ where $\step$ is the step-size of \eqref{eq:SGDA} for $k\in \{1,2,3,4\}$.
\end{lemma}
\begin{proof}
$\\$\textbf{Second moment. }
   We start by analyzing the second moment 
   \begin{align*}
       \norm{\state_{\time+1} -\state^*}^2 =&\norm{\state_\time -\step\op(\state_\time) -\step\noise_\time(\state_\time) -\state^*}^2\\
       \leq& \norm{\state_\time-\state^*} -2\step\inner{\op(\state_\time)}{\state_\time -\state^*} -2\step\inner{\noise_\time(\state_\time)}{\state_\time -\state^*}\\
&+2\step^2\lips^2\norm{\state_\time -\state^*} +2\step^2\norm{\noise_\time(\state_\time)}^2.
\end{align*}
We now apply the expectation and quasi strong monotonicity of the operator and get
\begin{align*}
\exof{\norm{\state_{\time+1} -\state^*}^2\given\filter_\time}\leq \norm{\state_\time-\state^*}^2(1+2\step^2\lips^2 -2\step\wqsmscale) + 2\step^2\sbound^2.
\end{align*}
By choosing $1+2\step^2\lips^2 -2\step\wqsmscale <1-\step\wqsmscale$ equivalently $\step<\dfrac{\wqsmscale}{2\lips^2}$ we have 
\begin{align*}
  \exof{\norm{\state_{\time+1} -\state^*}^2}&\leq \norm{\state_\start-\state^*}^2(1-\step\wqsmscale)^{\time+
1} + 2\step^2\sbound^2 \sum_{k=0}^\time(1-\step\wqsmscale)^k \\
&\leq \norm{\state_\start-\state^*}^2(1-\step\wqsmscale)^{\time+
1} +\dfrac{2\step^2\sbound^2}{\step\wqsmscale}\\
&\leq \norm{\state_\start-\state^*}^2(1-\step\wqsmscale)^{\time+
1} +\dfrac{2\step\sbound^2}{\wqsmscale}.
\end{align*}
Thus if $\state\sim \pi_{\step}$, where $\pi_\step$ is the invariant distribution of the iterates of \eqref{eq:SGDA} we have that 
\begin{equation*}
    \int_{\R^d}\norm{x-\state^*}^2\dd (\pi(x)) \leq 2\dfrac{\sbound^2\step}{\wqsmscale}
\end{equation*}
since $\lim_{\time\to\infty}{\state_\time} \sim\pi_\step$.
\\\textbf{Fourth moment. } For the fourth moment, similarly as before we have that    
    \begin{align}
       \norm{\state_{\time+1} -\state^*}^4 =& (\norm{\state_{\time+1} -\state^*}^2)^2 \nonumber\\
       =& (\snorm^2 -2\step\inner{\op(\state_\time)+\noise_\time(\state_\time)}{\state_\time -\state^*} +\step^2\norm{\op(\state_\time)+\noise_\time(\state_\time)}^2)^2  \nonumber\\
       =& \snorm^4 +4\step^2(\inner{\op(\state_\time)+\noise_\time(\state_\time)}{\state_\time -\state^*})^2 +\step^4\norm{\op(\state_\time)+\noise_\time(\state_\time)}^4  \nonumber\\
       &-4\step\snorm^2\inner{\op(\state_\time)+\noise_\time(\state_\time)}{\state_\time -\state^*} \nonumber\\
       &-4\step^3\norm{\op(\state_\time)+\noise_\time(\state_\time)}^2\inner{\op(\state_\time)+\noise_\time(\state_\time)}{\state_\time -\state^*}  \nonumber\\
       &+2\step^2\norm{\op(\state_\time)+\noise_\time(\state_\time)}^2\snorm^2 \nonumber\\
       \leq&\snorm^4 + 4\step^2\snorm^2(2\lips^2\snorm^2 +2\norm{\noise_\time(\state_\time)}^2) \label{eq:4mom_2}\\
       &+\step^4(8\lips^4\snorm^4 + 8\norm{\noise_\time(\state_\time)}^4)\label{eq:4mom_3} \\
       &-4\step\wqsmscale\snorm^4 -4\step\snorm^2\inner{\noise_\time(\state_\time)}{\state_\time-\state^*}\label{eq:4mom_4}\\
       &+4\step^3(4\lips^3\snorm^3 +4\norm{\noise_\time(\state_\time)}^3)\snorm \label{eq:4mom_5}\\
       &+4\step^2(\lips^2\snorm^4 + \norm{\noise_\time(\state_\time)}^2\snorm^2),
    \end{align}
    where we used in the second summand \cref{eq:4mom_2} the Cauchy-Schwarz inequality, Lipschitz continuity of the operator and the identinty $\norm{x+y}^2\leq 2\norm{x}^2 + 2\norm{y}^2$. For the third one \cref{eq:4mom_3} we used the identity $\norm{x+y}^4 \leq 8\norm{x}^4 + 8\norm{y}^4$, Lipschitzness of the operator. For the fourth one \cref{eq:4mom_4} we used the quasi strong monotonicity of the operator. For the firth one \cref{eq:4mom_5} we used Cauchy-Schwarz inequality and the identity $\norm{x+y}^2\leq 2\norm{x}^2 + 2\norm{y}^2$ and Lipschitzness of the operator. Thus in the \acl{RHS} of the above inequality we have constant terms, the $\snorm^4$, $\snorm^2$ and $\snorm$. Specifically, by rearranging we get
    \begin{align*}
        \norm{\state_{\time+1} -\state^*}^4 
       \leq& \snorm^4 (1+8\step^2\lips^2 + 8\step^4\lips^4 -4\step\wqsmscale + 16\step^3 \lips^3 +4\step^2\lips^2) \\
       &+\snorm^2(12\step^2\norm{\noise_\time(\state_\time)}^2) \\
       &+\snorm(16\step^3\norm{\noise_\time(\state_\time)}^3 -4\snorm^2\inner{\noise_\time(\state_\time)}{\state_\time -\state^*}\\
       &+8\step^4\norm{\noise_\time(\state_\time)}^4.
    \end{align*}
    Applying the expectation given the filtration $\filter_\time$  and setting $\bar{\lips} = \max\braces{\lips^2,\lips^3,\lips^4}$ we have
    \begin{align*}
\exof{\snorm^4\given\filter_\time} 
       \leq& \exof{\norm{\state_{\time+1} -\state^*}^4\given\filter_\time}(1+16\bar{\lips}(\step^2+\step^3+\step^4) -4\step\wqsmscale ) \\
&+\exof{\snorm^2\given\filter_\time}(12\step^2\sbound^2)\\
&+\exof{\snorm\given\filter_\time}(16\step^3{\kyrt}^3) + 8\step^4{\kyrt}^4.
    \end{align*}
    By choosing step-size such that 
    \begin{equation*}
        \left\{\begin{matrix}
 \step<1& \text{ for simplicity }\\ 
16\bar{\lips}(\step^2+\step^3 +\step^4) -4\step\wqsmscale<-2\step\wqsmscale &
\end{matrix}\right.
    \end{equation*}
    we have that 
    \begin{align*}
       \exof{\norm{\state_{\time+1} -\state^*}^4\given\filter_\time}(2\step\wqsmscale) \leq &\exof{\snorm^2\given\filter_\time}( 12\step^2\sbound^2) \\
       &+ \exof{\snorm\given\filter_\time}(16\step^3{\kyrt}^3) + 8\step^4{\kyrt}^4.
    \end{align*}
    Now consider $x\sim\pi_{\step}$ and  let $\ex_{x\sim\pi_\step}\bracks{\norm{x-\state^*}^k} = \text{Mom}(k)$. Notice that 
     the first moment is also bounded by $\bigoh(\sqrt{\step/\wqsmscale})$ since from \cref{eq:aux} and Lipschitzness of the operator we have
     \begin{equation*}
         \norm{\state_{\time+1}-\state^*}^2 \leq (1-2\wqsmscale\step+\step^2\lips^2)\snorm^2 +\norm{\noise_\time(\state_\time)}^2
     \end{equation*}
     Thus, combining all these we have
     \begin{align*}
         \text{Mom}(4)2\wqsmscale\step\leq \text{Mom}(2)\bigoh(\step^2) +\text{Mom}(1)\bigoh(\step^3) +\bigoh(\step^4).
     \end{align*}
     equivalently
     \begin{align*}
         \text{Mom}(4)\leq \text{Mom}(2)\bigoh(\step/\wqsmscale) +\text{Mom}(1)\bigoh(\step^2/\wqsmscale) +\bigoh(\step^3/\wqsmscale).
     \end{align*}
     But $\text{Mom}(2) \leq \bigoh(\step/\wqsmscale)$ and $\text{Mom}(1)\leq \bigoh(\sqrt{\step/\wqsmscale})$, thus
     \begin{equation*}
         \text{Mom}(4) \leq \bigoh(\step^2/\wqsmscale^2),
     \end{equation*}
     which implies that there exists $c \leq c_0\max\braces{{\kyrt}^3,{\kyrt}^4,\sbound,\sbound^2}$ such that 
     \begin{equation*}
         \text{Mom}(4) \leq c\step^2/\wqsmscale^2.
     \end{equation*}
     \end{proof}

\addcontentsline{toc}{subsubsection}{Proof of Theorem~\ref{thm:bias_expansion} (Richardson Extrapolation for Quasi-Monotone Operators)}
\begin{theorem}\label{thm:app-bias_expansion}[Restated \cref{thm:bias_expansion}]
    Suppose \cref{asm:solutionset,asm:lipschitz-lineargrowth,asm:wqsm,asm:oracle,asm:uniformbound,asm:fourth_noise} hold. There exists a threshold $\thres$ such that if  $\step\in(0,\thres),$ \eqref{eq:SGDA} admits unique stationary distribution $\pi$, that depends on the choice of step-size, and 
    \begin{equation} \label{eq:bias_expansion-app}
        \ex_{\point\sim\pi}\bracks{\point} -\state^* = \step\Delta(\state^*) + \bigoh(\step^2),
    \end{equation}
    where $\Delta(\state^*)$ is a vector independent of the choice of step-size $\step$.
\end{theorem}
\begin{proof}
    Let $\bar{\state} = \int_{\R^d}x\pi_\step(x)\dd x = \ex_{\point\sim\pi_{\step}}\bracks{\point} $ and let $\step < \min(\step_{\text{thresh}}^{\text{\ref{lem:moments}}},\step_{\text{thresh}}^{\text{\ref{thm:close-app}}}):=\theta'$ such that \cref{lem:moments} and \cref{thm:close-app} hold. Assume that we run \eqref{eq:SGDA}  $(\state_\time)_{\time\geq \start}$ and $\state_\start\sim \pi_{\step}$;  since the algorithm is initialized with the invariant distribution, then all the iterations inevitably follow the invariant distribution. We start by applying  Taylor expansion, on the operator, of second and third order around the solution $\state^*$
    \begin{equation}\label{eq:A}\tag{A}
        \op(\state) = \grad \op(\state^*)\odot [\state-\state^*] +\frac{1}{2}\grad^2\op(\state^*)\odot [\state-\state^*]^2+\text{Res}_3(\state),
    \end{equation}
    \vspace{-1ex}
\begin{equation}\label{eq:B}\tag{B}
        \op(\state) = \grad \op(\state^*)\odot[\state-\state^*]  + \text{Res}_2(\state),
    \end{equation}
    where $\text{Res}_2(x), \text{Res}_3(x)$ are the corresponding residuals of the Taylor expansion for which it holds that $\sup_{x\in\R^d} \braces{\norm{\text{Res}_3(x)}/\ssnorm^3} < \infty$ and $\sup_{x\in\R^d} \braces{\norm{\text{Res}_2(x)}/\ssnorm^2} < \infty$. Notice also that 
    \begin{equation}\label{eq:C}\tag{C}
       \int_{x\in\R^d} \text{Res}_3(x)\pi_\step(x)\dd x < c_3 \int_{x\in\R^d}\ssnorm^3\pi_\step(x)\dd x \leq c_3\text{Mom}(3) \leq \bigoh(\step^{3/2}), 
    \end{equation}
    \begin{equation}\label{eq:D}\tag{D}
        \int_{x\in\R^d}\text{Res}_2(x)\pi_\step(x)\dd x \leq c_2 \int_{x\in\R^d} \ssnorm^2 \pi_\step(x)\dd x \leq c_2 \text{Mom}(2) \leq \bigoh(\step).
    \end{equation}
    Additionally,  by definition of \eqref{eq:SGDA} we get that $\state_1 = \state_\start - \step \op(\state_\start) -\step \noise_\start(\state_\start)$.  Since $\state_0\sim\pi_\step$ we have that $\state_1\sim\pi_\step$ and thus we have
    \begin{equation*}
        \ex_{\state_1\sim\pi_\step}\bracks{\state_1} = \ex_{\state_\start\sim\pi_\step}\bracks{\state_\start} -\step \ex_{\state_\start\sim\pi_\step}\bracks{\op(\state_\start)} - \step \ex_{\state_\start\sim\pi_\step}\bracks{\noise_\start(\state_\start)},  
    \end{equation*}
    which implies that 
    \begin{equation}\label{eq:E}\tag{E}
        \ex_{x\sim\pi_\step}\bracks{\op(\state)} =0.
    \end{equation}
    With these equations at hand, we proceed and take the expectation of \eqref{eq:A} with respect to the invariant distribution, combining also \eqref{eq:C} and \eqref{eq:E} and we get 
    \begin{equation} \label{eq:useful-a}
        \grad\op(\state^*)\odot [\bar{\state}-\state^*] +\dfrac{1}{2}\int_{x\in\R^d}\grad^2\op(\state^*)\odot [\state-\state^*]^2 \pi_\step(x)\dd x = \bigoh(\step^{3/2}).
    \end{equation}
    Again we focus on the first update of \eqref{eq:SGDA} and we have
    \begin{align*}
        \state_1 &=\state_\start -\step\op(\state_\start) - \step\noise_\start(\state_\start)\\
      \state_1 -\state^*&=\state_\start-\state^* -\step\left( \grad \op(\state^*)\odot[\state_\start-\state^*]  + \text{Res}_2(\state_\start) \right) - \step\noise_\start(\state_\start)\\
      \state_1 -\state^*&=(\eye - \step(\op(\state^*) )\odot[\state_\start-\state^* ] - \step\text{Res}_2(\state_\start) - \step\noise_\start(\state_\start).
    \end{align*}
    We now compute $[\state_1-\state^*]^2 = (\state_1 -\state^*)(\state_1 -\state^*)^\top$ and apply the expectation with respect to the invariant distribution and the noise and we have
    \begin{align*}
        \ex_{x\sim\pi_\step}\bracks{[x-\state^*]^2
        } = (\eye -\step\grad\op(\state^*))\odot \ex_{x\sim\pi_\step}\bracks{(x-\state^*)^2}\odot(\eye -\step\grad\op(\state^*))+\step^2\ex_{x_0\sim\pi_\step}\bracks{[\noise_\start(\state_\start)]^2}\\+\bigoh\left(\underbrace{\step   \int_{x\in\R^d} \text{Res}_3(x)\odot (\eye - \step(\op(\state^*) )\odot[\state_\start-\state^* ] \pi_\step(x)\dd x+\gamma^2+\cdots}_{\gamma^{5/2}}\right).
    \end{align*}
This leads to
\begin{equation*}
    \ex_{x\sim\pi_\step}\bracks{[x-\state^*]^2}=\step Q(\state^*) \ex_{x_0\sim\pi_\step}\bracks{[\noise_\start(\state_\start)]^2} +\bigoh(\step^{3/2}), 
    \label{eq:useful-c}
\end{equation*}
where $Q(\state^*):=(\grad\op(\state^*)\odot \eye + \eye\odot \grad\op(\state^*)-\step \grad\op(\state^*)\odot\grad\op(\state^*))^{-1}$, which is invertible since 
\[
\grad\op(\state^*)\odot \eye + \eye\odot \grad\op(\state^*)-\step \grad\op(\state^*)\odot\grad\op(\state^*)=
\grad\op(\state^*)\odot M(\state^*) + M(\state^*)\odot \grad\op(\state^*),
\]
where $M(\state^*):= \eye-\step/2 \grad\op(\state^*)$. By quasi-monotonicity around $\state^*$ and by choosing $\gamma<\min(2L,\theta'):=\theta$ we get that the tensor $Q(\step^*)$ is positive definite tensor.

By applying a second-order Taylor expansion about $\state^*$ in $Op(\point):=[\noise_\run(\point)]^2$, and utilizing the same reasoning as above in combination with the differentiability of the noise tensor (see Assumption~\ref{asm:fourth_noise}), we derive the following:
\begin{align}\label{eq:useful-b}
    \ex_{\point\sim\pi_\step}\bracks{[\noise_\run(\point)]^2}&=[\noise_\run(\point^*)]^2 + \bigoh(\step)\\
    \ex_{\point\sim\pi_\step}\bracks{[\noise_\run(\point)]^2\odot [\point-\point^*]}&=[\noise_\run(\point^*)]^2\odot [\ex_{\point\sim\pi}\bracks{\point} -\point^*] + \bigoh(\step).
\end{align}

Combining \eqref{eq:useful-a},\eqref{eq:useful-c},\eqref{eq:useful-b}, we get
that
\begin{equation*} \bar{\state}-\state^* =-\dfrac{1}{2}[
            \grad\op(\state^*)]^{-1}\odot\grad^2\op(\state^*)\odot \big(\step Q(\state^*) \ex_{x_0\sim\pi_\step}\bracks{[\noise_\start(\state_\start)]^2} +\bigoh(\step^{3/2})\big)  + \bigoh(\step^{3/2}),
\end{equation*}
which implies that
\begin{equation*} \bar{\state}-\state^* =-\dfrac{1}{2}[
            \grad\op(\state^*)]^{-1}\odot\grad^2\op(\state^*)\odot \big(\step Q(\state^*) \odot 
            \{[\noise_\run(\point^*)]^2 + \bigoh(\step)\}
            +\bigoh(\step^{3/2})\big)  + \bigoh(\step^{3/2}),
\end{equation*}
or equivalently
\begin{equation*} \bar{\state}-\state^* =\gamma\Delta(\state^*)+\bigoh(\gamma^{3/2}).
\end{equation*}
$$$$$$$$

The rest of the proof has the goal to improve the last term the order to $\bigoh(\step^2)$.
\begin{enumerate}
    \item We have seen that via \eqref{eq:useful-c},\eqref{eq:useful-b},:
    $\ex_{x\sim\pi_\step}\bracks{[x-\state^*]^2}=\step Q(\state^*)\odot[U_t(\state^*)] + \step^2Q(\state^*)+o(\step^2)$
    \item With similar calculations we can prove that:
    $\ex_{x\sim\pi_\step}\bracks{[x-\state^*]^3}=\step^2 B(\state^*)+o(\step^2)$
\end{enumerate}

Using 4-th order taylor again we get the following equality
\begin{align*}
      \state_1 -\state^*&=\state_\start-\state^*\\
      &-\step\big( 
      \grad \op(\state^*)\odot [\state-\state^*] +\frac{1}{2!}\grad^2\op(\state^*)\odot [\state-\state^*]^2\\
      &\quad\quad\quad+
      \frac{1}{3!}\grad^3\op(\state^*)\odot [\state-\state^*]^2
      +\text{Res}_4(\state)
      \big)\\
       &- \step\noise_\start(\state_\start)
\end{align*}

Applying expectation in the above equality and combining the bounds (1.) and (2.), we have that
    \begin{equation} 
    \begin{Bmatrix}
        \grad\op(\state^*)\odot [\bar{\state}-\state^*] +\dfrac{1}{2}\grad^2\op(\state^*)\odot \ex_{x\sim\pi_\step}\bracks{[x-\state^*]^2} \\\\+\\\\
      \frac{1}{3!}\grad^3\op(\state^*)\odot\ex_{x\sim\pi_\step}\bracks{[x-\state^*]^3}
      +\ex_{x\sim\pi_\step}[\text{Res}_4(\state) ]
      \end{Bmatrix}= 0
    \end{equation}
By applying the fourth-moment bound for $\ex_{x\sim\pi_\step}[\text{Res}_4(\state) ]=\bigoh(\step^2)$ we get the promised result.
\end{proof}

\end{document}